\documentclass[11pt]{article}

\usepackage[utf8]{inputenc} % allow utf-8 input

\pdfoutput=1

\usepackage{algorithm}
\usepackage{algorithmic}
\usepackage{amsmath}
\usepackage{amssymb}
\usepackage{amsthm,amsfonts}
\usepackage{array}
\usepackage{authblk}
\usepackage{babel}
\usepackage{booktabs}
\usepackage{caption}
\usepackage{colortbl}% http://ctan.org/pkg/xcolor
\usepackage{enumitem}
\usepackage[T1]{fontenc}    
\usepackage{framed}
\usepackage{fullpage}
\usepackage{graphicx}
\usepackage[colorlinks, linkcolor=red, anchorcolor=blue, citecolor=blue]{hyperref}
\usepackage{indentfirst}
\usepackage[utf8]{inputenc}
\usepackage{lipsum}        
\usepackage{makecell}
\usepackage{mathrsfs}
\usepackage{mathtools}

\usepackage[framemethod=default]{mdframed}
\usepackage{microtype}     
\usepackage{multicol}
\usepackage{multirow}
\usepackage{natbib}
\usepackage{nicefrac}       % compact symbols for 1/2, etc.
\usepackage{tablefootnote}
\usepackage{url}            % simple URL typesetting
\usepackage{wrapfig,lipsum}
\usepackage{xcolor}         % colors
\usepackage{xspace}

\usepackage{doi}
\usepackage{tikz}
\usetikzlibrary{decorations.pathreplacing}
\usepackage{amsmath,amsfonts,bm}

\newcommand{\RR}{{\mathbb R}}

\def\1{\bm{1}}

\def\rvw{{\mathbf{w}}}
\def\rvx{{\mathbf{x}}}
\def\rvy{{\mathbf{y}}}
\def\rvz{{\mathbf{z}}}

\def\vzero{{\bm{0}}}

\def\vmu{{\bm{\mu}}}

\def\vm{{\bm{m}}}

\def\vv{{\bm{v}}}
\def\vw{{\bm{w}}}
\def\vx{{\bm{x}}}
\def\vy{{\bm{y}}}
\def\vz{{\bm{z}}}
\def\vmu{{\bm{\mu}}}
\def\vsigma{{\bm{\sigma}}}

\def\mA{{\bm{A}}}
\def\mB{{\bm{B}}}

\def\mH{{\bm{H}}}
\def\mI{{\bm{I}}}

\def\mQ{{\bm{Q}}}

\def\mT{{\bm{T}}}

\def\mV{{\bm{V}}}

\DeclareMathAlphabet{\mathsfit}{\encodingdefault}{\sfdefault}{m}{sl}
\SetMathAlphabet{\mathsfit}{bold}{\encodingdefault}{\sfdefault}{bx}{n}

\def\gI{{\mathcal{I}}}

\def\gL{{\mathcal{L}}}

\def\gP{{\mathcal{P}}}

\newcommand{\grad}{\ensuremath{\nabla}}

\newcommand{\R}{\mathbb{R}}

\newcommand{\der}{\mathrm{d}}

\DeclareMathOperator{\Tr}{Tr}

\makeatletter
\def\thickhline{%
  \noalign{\ifnum0=`}\fi\hrule \@height \thickarrayrulewidth \futurelet
   \reserved@a\@xthickhline}
\def\@xthickhline{\ifx\reserved@a\thickhline
               \vskip\doublerulesep
               \vskip-\thickarrayrulewidth
             \fi
      \ifnum0=`{\fi}}
\makeatother

\newlength{\thickarrayrulewidth}
\newtheorem{lemma}{Lemma}[section]
\newtheorem{theorem}[lemma]{Theorem}

\title{An Improved Analysis of Langevin Algorithms with Prior Diffusion for Non-Log-Concave Sampling}

\author[$\dagger$]{\normalsize Xunpeng Huang\thanks{Mail to \href{xhuangck@connect.ust.hk}{xhuangck@connect.ust.hk}, \href{dzou@cs.hku.hk}{dzou@cs.hku.hk}}}
\author[$\dagger$]{\ Hanze Dong}
\author[$\S$]{\ Difan Zou}
\author[$\ddag$]{\ Tong Zhang}

\affil[$\dagger$]{Hong Kong University of Science and Technology}
\affil[$\S$]{The University of Hong Kong}
\affil[$\ddag$]{University of Illinois Urbana-Champaign}

\begin{document}

\date{}
\maketitle

\begin{abstract}
Understanding the dimension dependency of computational complexity in high-dimensional sampling problem is a fundamental problem, both from a practical and theoretical perspective. 
Compared with samplers with unbiased stationary distribution, e.g., Metropolis-adjusted Langevin algorithm (MALA), biased samplers, e.g., Underdamped Langevin Dynamics (ULD), perform better in low-accuracy cases just because a lower dimension dependency in their complexities.
Along this line,~\citet{freund2022convergence} suggest that the modified Langevin algorithm with prior diffusion is able to converge dimension independently for strongly log-concave target distributions. 
Nonetheless, it remains open whether such property establishes for more general cases.
In this paper, we investigate the prior diffusion technique for the target distributions satisfying log-Sobolev inequality (LSI), which covers a much broader class of distributions compared to the strongly log-concave ones.
In particular, we prove that the modified Langevin algorithm can also obtain the dimension-independent convergence of KL divergence with different step size schedules.
The core of our proof technique is a novel construction of an interpolating SDE, which significantly helps to conduct a more accurate characterization of the discrete updates of the overdamped Langevin dynamics.
Our theoretical analysis demonstrates the benefits of prior diffusion for a broader class of target distributions and provides new insights into developing faster sampling algorithms.

%\dz{I am not sure whether using nearly half of the paragraph to motivate this work is a good idea. I guess we do not need to mention optimization. Perhaps we can go with (1) improving the dimension dependency in high-dimensional sampling problem is important; (2) There is a possible solution, but can only work for strongly log-concave distribution; (3) we extend to more general cases, LSI and proving a nearly dimension-independent bound; (4) mentioning the key step is to XXX.}
%\xunp{Revision Done}
%We extend this result for the target distribution log-Sobolev inequality

 %in terms of Kullback–Leibler (KL) divergence. 
%With log-Sobolev inequality, our algorithm 
  
 % By introducing a reweighted kernel, we provide a novel analysis to obtain the convergence rate of SVGD in the population limit without Stein log-Sobolev inequality.
\end{abstract}

\section{Introduction}

Sampling from unnormalized distribution aims to obtain the particles from a target density function $p_*\propto \exp(-U)$ on $\RR^d$, which is a fundamental task in statistics \citep{neal1993probabilistic}, scientific computing \citep{robert1999monte}, and machine learning \citep{bishop2006pattern}.
Monte Carlo Markov Chain (MCMC) \citep{gilks1995markov,brooks2011handbook} is a class of sampling algorithms to construct Markov chains, achieving the target distribution when it becomes stable. Many previous works concentrate on the Markov properties to investigate the behavior of MCMC algorithms. 

Langevin algorithms, a specific type of MCMC methods, offer a compelling approach to constructing Markov chains for sampling. These algorithms leverage the Stein score, which corresponds to the gradient of the logarithm of the density function, making them practical and efficient in real-world applications.
Thus, the variants of Langevin algorithms 
\citep{rossky1978brownian}
have gained significant popularity in the sampling literature, such as 
Unadjusted Langevin Algorithm (ULA) \citep{dalalyan2017theoretical,vempala2019rapid,durmus2019high}, Stochastic Gradient Langevin Dynamics (SGLD) \citep{welling2011bayesian,teh2016consistency}  and Metropolis-Adjusted Langevin Algorithm (MALA) \citep{roberts2002langevin,xifara2014langevin}. 
One intriguing aspect of continuous Langevin algorithms is their connection to the gradient flow of the Kullback-Leibler (KL) divergence 
 \citep{risken1996fokker}. This characteristic makes them valuable as first-order optimizers in probability spaces, especially when dealing with different subsampling and discretization strategies \citep{jordan1998variational,jordan1999introduction,wibisono2018sampling,ma2019there}. 

For a more in-depth analysis of optimization algorithms, the objectives are usually restricted to a certain family of functions, e.g., strong convexity or Polyak-Lojasiewicz (PL) inequality.
Similarly, in sampling problems, researchers often consider target distributions, denoted as $p_*$, that adhere to two types of assumptions: (1) log-concavity and (2) isoperimetric inequalities. 
The first case is presented as the convexity of $-\log p_*$, resulting in the KL divergence to $p_*$ being a strongly geodesic convex functional ~\citep{ambrosio2005gradient}.
For the second case, isoperimetric inequalities, such as Poincar\'e inequality and log-Sobolev inequality (LSI), encompass a broader range of distribution classes. 
Beyond the log-concavity, it also includes non-log-concave distributions, such as mixtures and perturbations of  log-concave distributions, which are of great interest in real applications.
Besides, ~\citet{wibisono2018sampling} explain LSI as the gradient-dominant condition or PL condition in a geodesic sense.
Although both of these two assumptions will lead to the linear convergence of KL divergence for continuous Langevin algorithms just like their counterpart in Euclidean space, compared with the variants of gradient descent in conventional optimization \citep{nesterov2018lectures}, the convergence rate of discrete Langevin algorithms
 usually has an additional $\mathcal{O}(d)$ dimensional dependency \citep{dalalyan2017theoretical,cheng2018convergence,ma2019sampling,dalalyan2019user,vempala2019rapid}. 
 The distinction between the convergence of sampling and optimization is so conspicuous and needs to be explored further.

To fill this gap, \citet{freund2022convergence} suggest that Langevin algorithms can achieve dimension-independent convergence rate for posterior sampling when the negative log-likelihood is convex and Lipschitz smooth (or Lipschitz continuous). With tractable strongly log-concave prior, the convergence rate of the modified Langevin dynamics\footnote{The diffusion term is solved with both Brownian motion and the tractable prior. When the prior is Gaussian, the diffusion term is the solution of Ornstein–Uhlenbeck (OU) process.} only depends on the trace of log-likelihood Hessian, rather than the explicit dimension $d$. In the ridge separable case, %\dz{possible reference?}
the trace of Hessian can be independent of the dimension $d$~\citep{freund2022convergence}. 
However, for more general target distributions beyond log-convexity, whether such a dimension-independent convergence will establish remains unknown.  
From a technical view, 
due to the lack of convexity ~\citep{durmus2019high,freund2022convergence} and irregular iterations 
%\dz{the phrase "two-stage iterations" is a bit confusing as it did not appear before}, 
it is more sophisticated to
control the bias deliberately for the inexact discretization of the continuous Langevin dynamics \cite{vempala2019rapid}. 

Therefore, in this paper, we attempt to answer the question
\begin{quote}
\centering
\emph{Can we achieve dimension-independent convergence rate for sampling from non-log-concave distributions?}
\end{quote}
In particular, we will focus on the distributions that satisfy LSI, a general condition for non-log-concave distributions that have been widely made in prior works~\citep{cheng2018convergence, vempala2019rapid, ma2019there}. 
%\dz{add some citations}. 
Additionally, we perform a slight modification  on the Langevin algorithm with prior diffusion considered in \citet{freund2022convergence} and investigate the convergence rate for this modified version. 
Technically, based on the equivalent transition kernel of the Markov Chain deduced by the iterations, we 
consider performing an analytic continuation for the two-stage discretized algorithm rather than relying on the geodesic convexity of KL divergence. 
By tracking the evolution of an analytic continuation of the density, we can obtain the 
largest possible step sizes by controlling the discretization-induced error.
The convergence rate of the proposed algorithm can be deduced subsequently.
Besides, we also introduce varying step sizes to improve the iteration complexity with a logarithmic factor.
The main contributions of this paper can be summarized in the following three perspectives:
\begin{itemize}[leftmargin=*]
    \item We show that under log-Sobolev inequality and proper regularity conditions, the KL convergence rate of Langevin dynamics with prior diffusion can be $\tilde{\mathcal{O}}\left(\mathrm{Tr}(\mH)\epsilon^{-1}\right)$ where the matrix $\mH$ depends on the property of $U$, rather than $\tilde{\mathcal{O}}\left(d\epsilon^{-1}\right)$ in ULA or $\tilde{\mathcal{O}}\left(-d\log^{\mathcal{O}(1)}(\epsilon)\right)$ in MALA. Compared with~\cite{freund2022convergence}, the target distribution in our work can go beyond strongly log-concave and thus covers a much larger class of non-log-concave target distributions.
    \item From a technical perspective, we innovatively construct an interpolate SDE to approximate the gradient flow, which helps to characterize a multi-stage update in overdamped Langevin dynamics more accurately and inspire the design of samplers, especially for composite and finite-sum potentials.
    Besides, compared with the specially designed step sizes in~\cite{freund2022convergence}, our technique can cover more general steps in practice.
    %\dz{also mention that our new technique can be used to cover more general stepsizes, in contrast to the specially designed step sizes in \cite{freund2022convergence}.}
    %\xunp{Revision Done}
    \item From an extension perspective, due to a good correspondence to the analysis in~\citep{vempala2019rapid} (one of the most flexible frameworks to analyze overdamped Langevin), our analysis can be easily generalized to a broader class of Langevin algorithms such as underdamped Langevin MCMC, SGLD, MALA, and even proximal samplers.
    %\xunp{Revision Done}
    %\dz{also proximal sampler?}.
\end{itemize}
%\xunp{Revision Done}
%\dz{we definitely need to strengthen the contribution section. For instance, we can summarize at least three contributions: (1) we develop a new Langevin algorithm and prove that its convergence rate can be independent of the explicit dimension parameter; 
%\xunp{Maybe we cannot claim that we propose a new algorithm because our algorithm is very similar to that in \citet{freund2022convergence}}
%(2) In comparison with \citet{freund2022convergence}, our proof technique is more general so that it can cover arbitrary learning rate scheduling, while in contrast, \citet{freund2022convergence} can only handle decreasing learning rate.
%(3) The proposed algorithm as well as analysis provides new insight in studying the convergence rate of sampling methods, and can be generalized to a broader class of Langevin algorithms such as underdamped Langevin MCMC, SGLD, and MALA.}
%\xunp{Both 2 & 3 are good points, we have modified in our paper.}
\section{Related Work}
\label{sec:related_work}
The dimension dependency of the computational complexity has been of interest to the sampling community for a long time.
In the following, we will divide the sampling algorithms into several categories and introduce their dimension dependency under different settings. 
%\dz{I think it's highly risky to claim low-accuracy samplers.. try to find different name. }
%\xunp{Revision Done}

\paragraph{Biased samplers.} Algorithms stemming from the discretization of stochastic processes, which have a stationary distribution denoted as $p_*$, including Langevin and underdamped Langevin diffusions, are the focus of this discussion. The bias inherent in the stationary distribution causes the overall runtime to scale with $\mathcal{O}(\epsilon^{-1})$. In the realm of Langevin diffusion, \citet{dalalyan2017further,dalalyan2019user} initially established a convergence rate for the Wasserstein 2 distance of $\tilde{\mathcal{O}}(d\epsilon^{-2})$ when $p_*$ exhibits strong log-concavity. Subsequently, \citet{cheng2018convergence,vempala2019rapid} achieved a convergence rate for the KL divergence of $\tilde{\mathcal{O}}(d\epsilon^{-1})$, under strong log-concavity and the Logarithmic Sobolev Inequality (LSI) conditions, respectively. For underdamped Langevin diffusion, \citet{shen2019randomized} reached a state-of-the-art convergence of the Wasserstein 2 distance at $\tilde{\mathcal{O}}(d^{1/3}\epsilon^{-2/3})$, applicable when $p_*$ is strongly log-concave. Furthermore, \citet{ma2019there} demonstrated an $\mathcal{O}(d^{1/2}\epsilon^{-1/2})$ KL convergence rate for underdamped Langevin diffusion, assuming LSI and high-order smoothness conditions for $p_*$. More recently, \citet{huang2024reverse} developed a novel diffusion-based Monte Carlo method that can provably sample from broader target distributions not limited by the LSI condition. This method's convergence rate was notably enhanced through the integration of a recursive score estimation technique, as detailed in \citet{huang2024faster}.

\paragraph{Unbiased samplers.} Such algorithms are typically designed in a way that the stationary distribution is unbiased w.r.t. the target distribution, which is achieved by introducing a Metropolis-Hastings filter to each step.
It makes the iteration complexity usually related to $\log(1/\epsilon)$ rather than $\epsilon^{-1}$.
However, the filter which debiases the algorithm also greatly complicates the analysis, and \citet{dwivedi2018log,chen2020fast} provide $\tilde{\mathcal{O}}(d\log^{O(1)}(1/\epsilon))$ Wasserstein 2 distance when $p_*$ is strongly log-concave. Recently,
\citet{altschuler2023faster} provides $\tilde{\mathcal{O}}(d^{1/2}\log^{O(1)}(1/\epsilon))$ KL convergence with proximal sampler reduction and even extend the assumption of $p_*$ to the LSI. 

\paragraph{Dimension-free samplers.} Such samplers usually expect the spectrum property of $-\log p_*$ to bring some benefit to the iteration complexity. 
When $p_*$ is strongly log-concave, A very recent work~\citep{freund2022convergence} find a Langevin algorithm to achieve an $\tilde{\mathcal{O}}(\Tr(\mH^{1/2}) \epsilon^{-1})$ dimension-independent  KL convergence rate where $\mH$ is an uniform upper bound of the Hessian matrix of $-\log p_*$. Besides, \citet{liu2024double} proved an $\tilde{\mathcal{O}}(\Tr(\mH)^{1/3} \epsilon^{-2/3})$ iteration complexity in Wasserstein distance for a double randomized ULD in the same setting, i.e., strongly log-concave.
The prior diffusion trick can be considered as re-balancing the coefficients of strong convexity for the energy term and the entropy term in~\citet{durmus2019analysis}. 
It is worth noting that the analysis along this line is similar to the techniques in conventional convex optimization ~\citep{durmus2019analysis} because the geodesic convexity in probability space is extremely similar to the convexity of objectives in  Euclidean space. 
However, unlike the first kind of proof, high dependency on geodesic convexity limits their extension to milder assumptions: the analysis under LSI can be hard.

In this work, our algorithm can be recognized as a dimension-free sampler while the analysis is more similar to biased samplers.
Specifically, we approximate the ideal gradient flow by controlling the discretization error along a constructed interpolate SDE.
This trick allows our analysis to be generalized to a broader class of Langevin algorithms such as underdamped Langevin MCMC, SGLD, and MALA.

%\dz{it would be better to categorize these prior works into multiple directions to make this section easier to follow. }
%\xunp{Revision Done.}
\section{Problem Setup and Preliminaries}
\label{sec:problem_setup}
This section will first introduce the notations commonly used and problem settings in the following sections. 
Then, we will show the main algorithm, i.e., Langevin Algorithm with Prior Diffusion as Alg~\ref{alg:lapd}, and the slight differences between our version and that provided in~\citet{freund2022convergence}.
Besides, we will also explain the intuitions about designing Alg~\ref{alg:lapd}.
After that, we will demonstrate the main difficulty when we expect to extend such a result to a broader class of target distributions.

\paragraph{Notations.} 
In this paper, we use letters $\vw$ and $\rvw$ to denote vectors in $\R^d$ and random variables in $\R^d$.
For any function, $f:\RR^d\rightarrow\RR$, $\nabla f(\cdot)$ and $\nabla^2 f(\cdot)$ denote the corresponding gradient and Hessian matrix, respectively. The Euclidean norm (vector) and its induced norm (matrix) are denoted by $\Vert\cdot\Vert$. For density $p$ and $q$, the KL divergence is denoted by $\mathrm{KL}(p\|q)=\mathbb{E}_p[\log (p/q)]$.
Besides, suppose $\rvw\sim p$ and $\rvw^\prime\sim q$, and let $\Gamma(p,q)$  be the class of joint distribution $\gamma(\vw,\vw^\prime)\colon \R^d\times \R^d \rightarrow \R$ satisfying $\int \gamma(\vw, \vw^\prime)\der \vw^\prime = p(\vw)$ and $\int \gamma(\vw, \vw^\prime)\der \vw = q(\vw^\prime)$ , the Wasserstein distance between $p$ and $q$ is defined as
\begin{equation*}
    W_2^2(p,q) = \min_{\gamma \in\Gamma(p,q)} \mathbb{E}_{\rvw,\rvw^\prime \sim \gamma}\left[\left\|\rvw-\rvw^\prime\right\|^2\right].
\end{equation*}
Specifically, we say $p_*$ satisfies LSI with constant $\alpha_*$ if for all smooth
function $\phi\colon \R^d \rightarrow \R$ with $\mathbb{E}_{p_*}[\|\rvw\|^2]\le \infty$,
\begin{equation*}
    \begin{aligned}
    \mathbb{E}_{p_*}\left[\phi^2\ln \phi^2\right]-\mathbb{E}_{p_*}\left[\phi^2\right]\ln \mathbb{E}_{p_*}\left[\phi^2\right]\le \frac{2}{\alpha_*} \mathbb{E}_{p_*}\left[\left\|\nabla \phi\right\|^2\right].
    \end{aligned}
\end{equation*}
By choosing $\phi^2=p/p_*$, we have
\begin{equation*}
        \begin{aligned}        2\alpha_* \mathrm{KL}\left(p\|p_*\right)\le \int p(\vw) \left\|\grad \log\frac{p(\vw)}{p_*(\vw)}\right\|^2 \der\vw.
        \end{aligned}
\end{equation*}

\paragraph{Problem setup.} 
When the dataset is given as $\vz$, we consider sampling from a posterior distribution over parameter $\vw\in \R^d$:
\begin{equation*}
    p(\vw | \vz) \propto p(\vz | \vw)\pi(\vw) \propto \exp\left(-U(\vw)\right) \coloneqq p_*(\vw).
\end{equation*}
The potential $U\colon \R^d \rightarrow \R$ can be decomposed as $ U(\vw) = f(\vw) + g(\vw)$, where $f,g\colon \R^d \rightarrow \R$. 
%\dz{I am not sure whether we can also highlight that our work is limited to the Bayesian setting but can handle all types of sampling problems as long as $U(w)$ can be written in the required form.}
This formulation is general in machine learning tasks. From a Bayesian view, we can divide the regularized empirical risk minimization loss to two parts: we can consider $f$ as the negative log-likelihood (loss function) and $g$ as the negative log-prior (regularization term).
Without loss of generality, we suppose $g$ is simple, and the SDE involving $g$ can be solved to high precision. 
Besides, the negative log-likelihood $f$ shares a similar setting with that in~\citep{freund2022convergence}.
\begin{enumerate}[label=\textbf{[{A}{\arabic*}]}]
    \item \label{ass:pri_g} The function $g$ is $m$-strongly convex and can be explicitly integrated. For simplicity, we choose $g(\vw)=(m/2)\left\|\vw\right\|^2$ in following sections.
    \item \label{ass:con_post_f} The function $f$ is second derivative, and for any $\vw\in\R^d$, it has
    \begin{equation*}
        \left\|\nabla^2 f(\vw)\right\|\preceq L\mI.
    \end{equation*}
    \item \label{ass:con_post_f2} For any $\vw\in\R^d$, we suppose
    \begin{equation*}
        \begin{aligned}
            \vzero\preceq \nabla^2 f(\vw) (\grad^2 f(\vw))^\top \preceq \mH^{\frac{1}{2}}\left(\mH^{\frac{1}{2}}\right)^\top = \mH \preceq L^2\mI,\quad \mathrm{where}\quad \mH^{\frac{1}{2}}\succeq \vzero.
        \end{aligned}
    \end{equation*}
    \item \label{ass:con_post_lsi} The posterior distribution, $p_* \propto \exp(-U)$, satisfies $\alpha_*$-log-Sobolev inequality (LSI).
\end{enumerate}
%\dz{point out that in the discussion section we provide some example data distributions.}
%\xunp{Revision Done}
By choosing the Gaussian prior, \citet{freund2022convergence} actually needs the strongly log-concave property for the posterior distribution $p_*$.
In our analysis, $p_*$ is only required to satisfy LSI, i.e.,~\ref{ass:con_post_lsi}, which is a standard assumption in recent sampling works~\citep{vempala2019rapid,wibisono2019proximal} and covers a wider class of measures than log-concave distributions~\cite{bakry1985diffusions}.
It can be preserved with bounded perturbation and Lipschitz mapping~\cite{vempala2019rapid}  where log-concavity fails. 
For~\ref{ass:con_post_f2}, due to the general Hessian condition under LSI (may not be positive semi-definite), we require the upper bound for the Hessian square, which is slightly stronger than that (only assuming upper bound of Hessian) in~\citet{freund2022convergence}.
However, similar to~\citet{freund2022convergence},~\ref{ass:con_post_f2} provides the potential for achieving a sharper convergence independent of the dimension number.
Specifically, in Section~\ref{sec:diss}, we demonstrate that \ref{ass:app_pri_g}-\ref{ass:con_post_lsi} can be verified when the target distribution is a Gaussian mixture.

%To justify the advantages, \citet{freund2022convergence} has provided several examples in Bayesian machine learning, which can also be handled in our settings due to the excellent property of~\ref{ass:con_post_lsi} and the establishment of~\ref{ass:con_post_f2}.

\paragraph{Langevin Algorithm with Prior Diffusion (LAPD).}
We show a variant of LAPD in Alg~\ref{alg:lapd}.
It should be noted that LAPD was first proposed by~\citep{freund2022convergence}.
Compared with that in~\citet{freund2022convergence}, the main differences lay in two aspects.
First, the order of stages is swapped in our variant, which only leads to the SDE stage at the first iteration being omitted.
Second, the output of Alg~\ref{alg:lapd} is $\tilde{\rvw}_T$ which violates the theoretical analysis in~\citet{freund2022convergence}.
The intuition of designing LAPD is the alternative iteration in composition optimization.
Specifically, for any distribution $p$, we decompose KL divergence between $p$ and the target distribution $p_*$, as follows
\begin{equation}
    \label{eq:kl_compos_func}
    \mathrm{KL}\left(p\|p_*\right) = \underbrace{\mathbb{E}_{\rvw\sim p}\left[f(\rvw)\right]}_{\mathrm{Term\ 1}} + \underbrace{\mathbb{E}_{\rvw\sim p}\left[g(\rvw)+\log p(\rvw)\right]}_{\mathrm{Term\ 2}}.
\end{equation}
According to~\ref{ass:pri_g}, the log-prior $g$ is relatively simple, e.g., quadratic functions, which means its SDE can be solved with high precision or even with a closed form.
Thus, the exact gradient flow of $\mathrm{Term\ 2}$
in Wasserstein space leads to exponential decay even though discretization error of $\mathrm{Term\ 1}$ will dominate the decrease of the total functional.

\paragraph{Main difficulties.} According to the Alg~\ref{alg:lapd}, we introduce some other notations.
Given a random variable $\tilde{\rvw}_0\sim \tilde{p}_0$, we would like to update $\tilde{\rvw}_k\sim \tilde{p}_k$, $t=1,\cdots,T$. 
From a composition optimization perspective, to get an explicit bound on the convergence of the sequence $\{f(\vw_k)\}_{k\in \mathbb{N}_+}$, deduced by the gradient descent type iteration, e.g.,
\begin{equation*}
    \vw_{k+1} = \vw_k - \eta_{k+1}\grad f(\vw_k) + \eta_{k+1}\xi(\vw_k),
\end{equation*}
to the minima $f(\vw_*)$,  one possibility (see~\citet{beck2009A}) is to upper bound function value differences by the distance among variables: 
%\dz{I am not sure whether we need to go with such as detailed derivation in the problem setup section.}
%\xunp{Revision Done}
\begin{equation*}
    \begin{aligned}
        2\eta_{k+1}\left(f(\vw_{k+1})-f(\vw_*)\right) \le \left\|\vw_k-\vw_*\right\|^2 - \left\|\vw_{k+1}-\vw_*\right\|^2 + C\eta_{k+1}^2,
    \end{aligned}
\end{equation*}
which highly depends on the convexity of function $f$.
Similarly, the convergence of Eq.~\ref{eq:kl_compos_func} requires
\begin{equation*}
    \begin{aligned}
    \mathrm{Term\ 1}\rightarrow \mathbb{E}_{\rvw\sim p_*}\left[f(\rvw)\right]\quad \mathrm{and}\quad \mathrm{Term\ 2}\rightarrow \mathbb{E}_{\rvw\sim p_*}\left[g(\rvw)+\log p(\rvw)\right].
    \end{aligned}
\end{equation*}
To upper bound the difference between $\mathrm{Term\ 1}$ and its limit with the Wasserstein 2 distance (corresponds to $L_2$ norm in Euclidean space), a crucial assumption is a geodesic convexity of $\mathrm{Term\ 1}$, which can be deduced by the convexity of $f$ in the target distribution $p_*$.
While such a convexity will be violated when $p_*$ only satisfies LSI.
Specifically, LSI is just like Polyak-Lojasiewicz (PL) inequality in a geodesic sense~\citep{wibisono2018sampling} due to the following equation
\begin{equation*}
    2\alpha_* \mathrm{KL}\left(p\|p_*\right)\le \left\|\mathrm{grad}_p \mathrm{KL}\left(p\|p_*\right) \right\|^2_p \coloneqq \int p(\vw) \left\|\frac{\delta \mathrm{KL}\left(p\|p_*\right)}{\delta p}\right\|^2 \der\vw.
\end{equation*}
However, unlike general convergence analysis with PL inequality, the functional Eq.~\ref{eq:kl_compos_func} does not have any property in Wasserstein space which is similar to the smoothness in Euclidean space.
Hence, there is no existing mature solution to extend the analysis of the optimization problem under PL inequality to the sampling problem under LSI inequality except for tracking the dynamics of SDE directly.
On the other hand, we will readily observe that, different from ULA~\citet{vempala2019rapid} described by one single SDE, the two-stage iteration of Alg.~\ref{alg:lapd} makes it challenging to present the dynamic of the particles with one It\^o's SDE.
We defer the solution in the next section.

\begin{minipage}{0.49\columnwidth}
    \begin{algorithm}[H]
    \caption{Langevin Algorithm with Prior Diffusion (LAPD)}
    \label{alg:lapd}
    \begin{algorithmic}[1]
            \STATE {\bfseries Input:} Initialization $\tilde{p}_0$, step size $\{\eta_k\}_{k\in\mathbb{N}}$, target $p_*\propto\exp(-f(x)-g(x))$.
            \STATE Draw $\tilde{\rvw}_0$ from $\tilde{p}_0$.
            \FOR{$k = 1$ to $T$}
                 \STATE{\bfseries Stage1:} Let 
                 \begin{equation*}
                    \small
                    \rvw_{k} = \tilde{\rvw}_{k-1} - \tilde{\eta}_{k-1}\grad f(\tilde{\rvw}_{k-1}).
                \end{equation*}
                \STATE{\bfseries Stage2:} Sample $\tilde{\rvw}_k$ from $\tilde{\rvw}(\eta_k)$ denoted by the following SDE
                \begin{equation*}
                    \small
                    \begin{aligned}
                        \tilde{\rvw}(\eta_k) = & \rvw_{k} -\int_0^{\eta_k} \grad g(\tilde{\rvw}(s))\der s \\
                        &+ \sqrt{2}\int_0^{\eta_k}\der \mB_t.
                    \end{aligned}
                \end{equation*}
                 \ENDFOR
                 \STATE {\bfseries Return:}  $\tilde{\rvw}_T$
    \end{algorithmic}
\end{algorithm}
\end{minipage}
\begin{minipage}{0.49\columnwidth}
    \begin{flushright}
    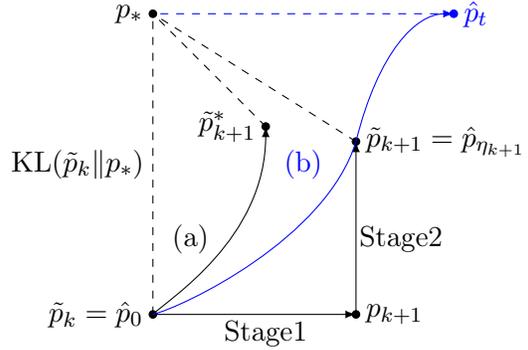
\begin{figure}[H]
\begin{tikzpicture}
\filldraw (0,0) circle (0.05) node[left]{$\tilde{p}_{k}=\hat{p}_0$};
\filldraw (0,4) circle (0.05) node[left]{$p_*$};
\draw[dashed] (0,0)--(0,4);
\node at (-1, 2) { $\mathrm{KL}(\tilde{p}_{k}\|p_*)$};
\node at (0.5, 1) { (a)};
\node[blue] at (2, 2) { (b)};
\node at (1.5, -0.25) { Stage1};
\node at (3.3, 1) { Stage2};
\filldraw (1.5,2.5) circle (0.05) node[left]{$\tilde{p}^*_{k+1}$};
\draw[-latex] (0,0)..controls (0.2, 0.2)and(1.5, 1)..(1.5,2.5);
\draw[dashed] (1.5,2.5)--(0,4);
\filldraw (2.7,2.3) circle (0.05) node[right]{$\tilde{p}_{k+1}=\hat{p}_{\eta_{k+1}}$};
\filldraw (2.7, 0) circle (0.05) node[right]{$p_{k+1}$};
\draw[-latex] (0,0)--(2.7,0);
\draw[-latex] (2.7,0)--(2.7,2.3);
\draw[dashed] (2.7,2.3)--(0,4);
\draw[blue] (0,0)..controls (0.4, 0.1)and(2.3, 1)..(2.7,2.3);
\draw[-latex,blue] (2.7,2.3)..controls (3, 3.5)and(3.5,4)..(4,4);
\filldraw[blue] (4, 4) circle (0.05) node[right]{$\hat{p}_{t}$};
\draw[dashed, blue] (4,4)--(0,4);
\end{tikzpicture}
\caption{An illustration for the proof of Lemma~\ref{lem:con_kl_contraction}. In each iteration, we compare the evolution of
(a) the continuous-time KL divergence flow for time $\eta_k$, and (b) the constructed SDE equivalent to a two-stage update at time $\eta_k$.} \label{fig:M1}
\end{figure}
    \end{flushright}
\end{minipage}

\section{Theoretical Results}
In this section, we will show the dimension-independent convergence rate of Alg~\ref{alg:lapd} with different step size schedules, e.g., fixing step sizes and varying step sizes, and compare them with previous work.
Besides, we will provide a roadmap for our analysis, which helps to understand our theoretical results.
Due to space limitations, we defer details in the Appendix.

\paragraph{KL convergence with a fixed step size.}
Our main result for Alg~\ref{alg:lapd} can be stated as follows.
\begin{theorem}
    \label{thm:kl_con_fx}
    Assume the posterior $p_*$ satisfies~\ref{ass:pri_g}-\ref{ass:con_post_lsi},and $\alpha_*\le L$ without loss of generality, we suppose the step size $\eta_k =\eta \in (0, \hat{\eta}]$ for $k\in\{1,2,\ldots T\}$ where
    \begin{equation*}
        \hat{\eta}\coloneqq \min\left\{(8m)^{-1}, (8\Tr(\mH^{\frac{1}{2}}))^{-1}, (1.5\alpha_*)^{-1}, \frac{\alpha_*}{8\sqrt{2}L^2}, \frac{\alpha_*\epsilon}{64\mathrm{Tr}(\mH)}\right\},
    \end{equation*}
    and $\left(m\tilde{\eta}_{k-1}\right)/\left(e^{m\eta_k}-1\right)=1$, then Alg~\ref{alg:lapd} satisfies
    \begin{equation*}
        \mathrm{KL}(\tilde{p}_k\|p_*)\le e^{-\alpha_*\eta k}\mathrm{KL}(\tilde{p}_0\|p_*) + \frac{32\eta}{\alpha_*}\cdot \mathrm{Tr}(\mH).
    \end{equation*}
    In this condition, requiring
    \begin{equation*}
    k\ge \frac{1}{\alpha_* \eta}\log \frac{2\mathrm{KL}(\tilde{p}_0\|p_*)}{\epsilon},
    \end{equation*}
    we have $\mathrm{KL}(\tilde{p}_k\|p_*)\le \epsilon$.
\end{theorem}
It should be noted that the step size $\eta_k$ will be selected as $\alpha_*\epsilon/(64\Tr(\mH))$ with $\epsilon$ diminishing.
In this condition, by requiring
\begin{equation*}
    k\ge \frac{1}{\alpha_* \hat{\eta}}\log \frac{2\mathrm{KL}(\tilde{p}_0\|p_*)}{\epsilon} \approx \frac{64\Tr(\mH)}{\alpha_*^2 \epsilon} \log \frac{2\mathrm{KL}(\tilde{p}_0\|p_*)}{\epsilon},
\end{equation*}
we have $\mathrm{KL}(\tilde{p}_0\|p_*)\le \epsilon$.
This theorem differs from the main results of~\citet{freund2022convergence} from two perspectives.
First, in~\citet{freund2022convergence}, the convergence is about the weighted average of KL divergence at different iterations, which does not correspond to some underline distribution obviously though the weighted average can be achieved with additional steps.
However, Theorem~\ref{thm:kl_con_fx} denotes that dimension-independent convergence can be obtained with only respect to the distribution of $\tilde{\rvw}_T$, and the particle of the last iteration can be output directly.
This property makes the results more clear and practical.
Second, the dimension-free convergence of KL divergence is required decreasing step sizes in~\citet{freund2022convergence}, while Theorem~\ref{thm:kl_con_fx} denotes that fixed step sizes can also achieve a similar result with an additional logarithmic term.

\begin{table*}
    \small
    \centering
    \begin{tabular}{cccc}
    \toprule
     Reference & Assumption & Iteration Complexity \\
     \midrule
     \citet{cheng2018convergence} & \ref{ass:pri_g},~\ref{ass:con_post_f},Convexity of $f$ & $\tilde{\mathcal{O}}\left(\frac{L^2}{m^2}\cdot \frac{d}{\epsilon}\right)$\\
     \citet{freund2022convergence} & \ref{ass:pri_g}-\ref{ass:con_post_f2}, Convexity of $f$ & $\mathcal{O}\left(\frac{L\mathrm{Tr}(\mH^{\frac{1}{2}})}{m^2 \epsilon}+ \frac{U(0)}{\epsilon} \right)$\\
     \midrule
     \citet{vempala2019rapid} & Smoothness of $U$ , \ref{ass:con_post_lsi} & $\tilde{\mathcal{O}}\left(\frac{L^2}{\alpha_*^2}\cdot \frac{d}{\epsilon}\right)$\\
    Theorem~\ref{thm:kl_con_fx} & \ref{ass:pri_g}-\ref{ass:con_post_lsi} & $\tilde{\mathcal{O}}\left(\frac{\mathrm{Tr}(\mH)}{\alpha_*^2 \epsilon}\right)$\\
    \midrule
    \citet{dalalyan2019user} & \ref{ass:pri_g},\ref{ass:con_post_f}, Convexity of $f$ & $\mathcal{O}\left(\frac{L^2}{m^3}\cdot \frac{d}{\epsilon}\right)$\\
    Theorem~\ref{thm:kl_con_vary} & \ref{ass:pri_g}-\ref{ass:con_post_lsi} & $\mathcal{O}\left(\frac{\mathrm{Tr}(\mH)}{\alpha_*^3 \epsilon}\right)$\\
     \bottomrule
\end{tabular}
    \caption{\small Comparison with previous results on overdamped Langevin algorithm. The first four lines utilize fixed step sizes and consider the KL convergence. The last two lines use varying step sizes, and consider the Wasserstein 2 distance convergence (with additional $m$ or $\alpha_*$ factor in the denominator). For the convergence of ULD and proximal samplers, we omit the detailed comparison here for ease of presentation due to their explicit dependency on dimension $d$ shown in Section~\ref{sec:related_work}.}
    \label{tab:comp_old}
\end{table*}

\paragraph{KL convergence with a varying step size.}
To get rid of the logarithmic terms in the number of iterations required to achieve the precision level $\epsilon$, we introduce varying step sizes with the inspire of~\citet{dalalyan2019user}.
Hence, we have the following result.
\begin{theorem}
    \label{thm:kl_con_vary}
    Consider Alg~\ref{alg:lapd} with varying step size $\eta_{k+1}$ defined by
    \begin{equation*}
        \eta_{k+1}=\frac{8\hat{\eta}}{9+3(k-K_0)_+\hat{\eta} \alpha_*}\quad \mathrm{and}\quad \left(m\tilde{\eta}_{k-1}\right)/\left(e^{m\eta_k}-1\right)=1
    \end{equation*}
    where $\hat{\eta}$ is denoted as
    \begin{equation*}
        \hat{\eta}\coloneqq \min\left\{ (8m)^{-1}, (8\Tr(\mH^{\frac{1}{2}}))^{-1}, (1.5\alpha_*)^{-1}, \frac{\alpha_*}{8\sqrt{2}L^2}\right\}
    \end{equation*}
    and $K_0$ denotes the smallest non-negative integer satisfying
    \begin{equation*}
        T_0\ge \frac{9}{8\hat{\eta}\alpha_*}\ln\left(\frac{\mathrm{KL}\left(\tilde{p}_0\|p_*\right)\alpha_*}{123\hat{\eta}\mathrm{Tr}(\mH)}\right).
    \end{equation*}
    If the target $p_*$ satisfies \ref{ass:pri_g}-\ref{ass:con_post_lsi} and $\alpha_*\le L$, for every $k\ge T_0$, we have
    \begin{equation*}
        \mathrm{KL}\left(\tilde{p}_{k}\|p_*\right)\le \frac{2^{10}\cdot \mathrm{Tr}(\mH)}{27L^{1.5}\alpha_*+6\left(k-T_0\right)\alpha^2_*}.
    \end{equation*}
\end{theorem}
The varying step sizes  has two important advantages as compared to the fixed step sizes.
The first is the independence of the target precision level $\epsilon$.
The second advantage is saving the logarithmic terms in Theorem~\ref{thm:kl_con_fx}.
 Actually, it suffices $T = T_0 + 2^9\cdot \mathrm{Tr}(\mH)/(3\epsilon \alpha^2_*)$ iterations to get $\mathrm{KL}\left(\tilde{p}_T\|p_*\right)\le \epsilon$.
More details of this theorem are provided in Appendix~\ref{appsec:main_results}.
We also provide Table~\ref{tab:comp_old} to show the comparison of assumptions and iteration complexity with several previous work. 
%\xunp{Revision Done}
%\dz{comparison with previous works perhaps need to be strengthened, as this is perhaps the most important part from the reviewer's views.}

\subsection{Proof sketch}
In this section, we expect to highlight the technical novelties by introducing the roadmap of our analysis.
Due to space limitations, we leave the technical details in Appendix~\ref{appsec:main_results}.

In Section~\ref{sec:problem_setup}, it has demonstrated that for functionals without the geodesic convexity in Wasserstein space, we can hardly upper-bound them by a global metric, e.g., the Wasserstein 2 distance between the current and the target distribution.
From another perspective, different from standard ULA~\citep{dalalyan2017theoretical,vempala2019rapid}, it is not easy to directly find an It\^o's SDE corresponding to the two-stage update of particles.
Therefore, the contraction of KL divergence with Alg~\ref{alg:lapd} remains a mystery in isoperimetry, e.g., log-Sobolev assumption.
To overcome these problems, for each iteration, we utilize a stochastic process $\{\hat{\rvw}_t\}_{t\ge 0}$ deduced by a constructed SDE and required to exist some random variable $\hat{\rvw}_{\eta_k}$ shares the same distribution with the particles after the two-stage update in Alg~\ref{alg:lapd}.
Then, by controlling the difference between such an SDE and the standard gradient flow of KL divergence, we may obtain the discretization error bounded by both $\tilde{\eta}_k$ and $\eta_k$ and prove the contraction of KL divergence in Alg~\ref{alg:lapd}.
Hence, there are two key steps in our analysis:
(1). Find the equivalent It\^o SDE for each two-stage update iteration. (2). Control the error between the constructed It\^o SDE and standard KL divergence flow; see Fig~\ref{fig:M1} for an illustration.

\paragraph{Construction of Equivalent SDE.}
From Alg~\ref{alg:lapd}, the transformation of density function from $\tilde{\rvw}_{k-1}$ to $\tilde{\rvw}_k$ can be considered as a combination of the change of variables (if the step size $\tilde{\eta}_{k-1}$ is small enough) and an Ornstein–Uhlenbeck process (if we suppose $g(\vw)=(m/2)\left\|\vw\right\|^2$).
With the Green's function of OU processes, the probability density of $\tilde{\rvw}_k$ conditioning on $\tilde{\rvw}_{k-1}=\vw_0$ as
\begin{equation*}
     \begin{aligned}
     p(\vw | \vw_0) = \left(\frac{2\pi (1-e^{-2m\eta_k})}{m}\right)^{-d/2} \cdot \exp \left[-\frac{m}{2}\cdot \frac{\left\|\vw-\left(\vw_0 - \tilde{\eta}_{k-1} \grad f(\vw_0)\right) e^{-m\eta_k}\right\|^2}{1-e^{-2m\eta_k}}\right].
     \end{aligned}
\end{equation*}
Besides, various literature, e.g.,~\cite{oksendal2013stochastic} shows that solutions to It\^o SDE are Markov processes.
In a probabilistic sense, it means that all It\^o processes are completely characterized by the transition densities from $\hat{\rvw}_s$ to $\hat{\rvw}_t$, i.e., $p(\vw_t, t|\vw_s, s)$.
Meanwhile, the transition density is also a solution to the Kolmogorov forward equation with a Dirac delta initial density concentrated on $\hat{\rvw}_s$ at time $s$.
Hence, conditioning on $\hat{\rvw}_0 \coloneqq \tilde{\rvw}_{k-1}=\vw_0$, if there is some transition density $p_{\vw_0}(\vw, t|\vw^\prime, t^\prime)$, solving some It\^o SDE, to satisfy 
\begin{equation*}
    p_{\vw_0}(\vw, \eta_k|\vw_0, 0)=p(\vw | \vw_0)
\end{equation*}
for any $\vw\in\R^d$, we consider its corresponding SDE equivalent to the two-stage update in Alg~\ref{alg:lapd}. 
Then, such an equivalent SDE can be constructed by the following lemma.
\begin{lemma}
    \label{lem:con_equivalent_sde}
    Conditioning on $\tilde{\rvw}_{k-1}=\vw_0$, there is a Markov process $\{\hat{\rvw}_t\}_{t\ge 0}$ deduced by the following SDE 
    \begin{equation}
        \label{eq:sde_4.1}
    \der \hat{\rvw}_t = -\left(m\hat{\rvw}_t+\frac{m\tilde{\eta}_{k-1}}{e^{m\eta_k}-1}\grad f(\vw_0)\right)\der t + \sqrt{2}\der \mB_t.
    \end{equation}
Then $\hat{\rvw}_{\eta_k}\sim p_{\vw_0}(\cdot, \eta_k|\vw_0,0)$ shares the same distribution as $\tilde{\rvw}_k$ because 
\begin{equation*}
    \begin{aligned}
        p_{\vw_0}(\vw, \eta_k|\vw_0, 0)=\left(\frac{2\pi (1-e^{-2m\eta_k})}{m}\right)^{-d/2} \cdot \exp \left[-\frac{m}{2}\cdot \frac{\left\|\vw-\left(\vw_0 - \tilde{\eta}_{k-1} \grad f(\vw_0)\right) e^{-m\eta_k}\right\|^2}{1-e^{-2m\eta_k}}\right].
    \end{aligned}
\end{equation*}
\end{lemma}
We leave more detailed proof of Lemma~\ref{lem:con_equivalent_sde} in Appendix~\ref{sec:implictsde} and the connection between Kolmogorov equations and transition density in Appendix~\ref{appsec:bke}. 
Then, for any suitable probability density function $\tilde{p}_{k-1}$ of $\tilde{\rvw}_{k-1}$, we set
\begin{equation*}
    \hat{p}_t(\vw)\coloneqq \int_{\vw_0} p_{\vw_0}(\vw, t|\vw_0, 0)\cdot \tilde{p}_{k-1}(\vw_0)\der \vw_0,
\end{equation*}
where $p_{\vw_0}$ is the transition density solving SDE.~\ref{eq:sde_4.1}, and
\begin{equation*}
    \begin{aligned}
        \hat{p}_{t|0}(\vw|\vw_0) \coloneqq & p_{\vw_0}(\vw, t|\vw_0, 0), \quad \hat{p}_{0t}(\vw_0, \vw) \coloneqq & \tilde{p}_{k-1}(\vw_0)\cdot p_{\vw_0}(\vw, t|\vw_0, 0)
    \end{aligned}
\end{equation*}
for abbreviation.
Follows from Lemma~\ref{lem:con_equivalent_sde}, we have $\hat{p}_{\eta_k}(\vw)=\tilde{p}_k(\vw)$ for any $\vw\in\R^d$.
It also implies
\begin{equation}
    \label{eq:con_sde_equi}
    \begin{aligned}
        \mathrm{KL}\left(\hat{p}_{\eta_k} \| p^*\right) = &\mathrm{KL}\left(\tilde{p}_{k}\| p^*\right)\quad \mathrm{and}\quad \mathrm{KL}\left(\hat{p}_0 \| p^*\right) = &\mathrm{KL}\left(\tilde{p}_{k-1}\| p^*\right).
    \end{aligned}
\end{equation}
Therefore, the contraction of KL divergence for each iteration of Alg~\ref{alg:lapd} can be obtained by tracking the conditional dynamics shown in SDE.~\ref{eq:sde_4.1}.
\paragraph{Control the variance of particles for each iteration.} Combining Lemma~\ref{lem:con_equivalent_sde} and Eq.~\ref{eq:con_sde_equi}, the dynamic of KL divergence can be upper bounded as 
\begin{equation*}
    \begin{aligned}
        \frac{\der \mathrm{KL}\left(\hat{p}_t\|p_*\right)}{\der t} \le -\frac{3\alpha_*}{2}\cdot\mathrm{KL}(\hat{p}_t\|p_*)
        + \underbrace{\int \hat{p}_{0t}\left(\vw_0, \vw\right)\left\|\grad f(\vw_0)-\grad f(\vw)\right\|^2 \der (\vw_0, \vw)}_{\mathrm{discretization\ error}}
    \end{aligned}
\end{equation*}
with an appropriate choice of step sizes when the target distribution satisfies~\ref{ass:con_post_lsi}.
Actually, the discretization error is from the difference between SDE~\ref{eq:sde_4.1} and the standard Langevin dynamics
\begin{equation*}
    \der \hat{\rvw}_t = -\left(m\hat{\rvw}_t+\grad f(\hat{\rvw}_t)\right)\der t + \sqrt{2}\der \mB_t.
\end{equation*}
To control the error, we introduce the mean of particles along SDE~\ref{eq:sde_4.1} at time $t$ given $\vw_0$, i.e.,
\begin{equation*}
    \overline{\vw}(t, \vw_0)=e^{-mt}\vw_0 - \frac{1-e^{-mt}}{m} \grad f(\vw_0),
\end{equation*}
The discretization error can be divided into two parts, relating to the bias of the particles' mean, i.e.,
\begin{equation}
    \label{ineq:con_meanerr}
    \int \hat{p}_0(\vw_0) \left\|\grad f(\overline{\vw}(t,\vw_0)) -\grad f(\vw_0)\right\|^2 \der\vw_0,
\end{equation}
and their variance, i.e.,
\begin{equation}
    \label{ineq:con_varerr}
    \int \hat{p}_0(\vw_0) \int \hat{p}_{t|0}(\vw|\vw_0)\left\|\grad f(\vw)-\grad f(\overline{\vw}(t,\vw_0))\right\|^2\der \vw \der\vw_0,
\end{equation}
at time $t$ along SDE~\ref{eq:sde_4.1}.
The key point for LAPD to get rid of the dependence on the dimension number $d$ is that the difference in gradients here (Eq.~\ref{ineq:con_meanerr} and Eq.~\ref{ineq:con_varerr}) is w.r.t. the function $f$ whose Hessian is upper bounded by the semi-positive definite $\mH^{1/2}$ satisfying $\Tr(\mH)\ll d$.
However, in standard ULA, the discretization error is related to the difference of gradients w.r.t. $U$ containing the quadratic $g$, which leads to $\Tr(\nabla^2 U)=\Omega(d)$. 
Then, with the definition of particles' mean and smooth properties (\ref{ass:con_post_f} and~\ref{ass:con_post_f2}), Eq.~\ref{ineq:con_meanerr} can be controlled by $\mathcal{O}(\eta_k^2)$.
Besides, due to the Gaussian-type distribution $\hat{p}_{t|0}$, Eq.~\ref{ineq:con_varerr} only has the following upper bound
\begin{equation*}
     \frac{1-e^{-2mt}}{m}\cdot\mathrm{Tr}(\mH) = \mathcal{O}(\eta_k),
\end{equation*}
which dominates the discretization error. 
Such a result can be summarized as the following lemma about the contraction of KL divergence along SDE~\ref{eq:sde_4.1}.
%\xunp{Revision Done}
%\dz{the illustration from the first equation on this page to the following lemma is a bit tedious. Perhaps try to make this part more concise. Try to use a few sentences to explain why the discretization bound will be related to $\mathrm{Tr}(\mathbf H)$.}
\begin{lemma}
    \label{lem:con_kl_contraction}
    Suppose the target distribution $p_*$ satisfies~\ref{ass:pri_g}-~\ref{ass:con_post_lsi}, if we set
    \begin{equation*}
        \left(m\tilde{\eta}_{k-1}\right)/\left(e^{m\eta_k}-1\right)=1
    \end{equation*}
    and
    \begin{equation}
        \label{ineq:step0_upb}
        0< \eta_k\le \min\left\{ (8m)^{-1}, (8\Tr(\mH^{\frac{1}{2}}))^{-1}, (1.5\alpha_*)^{-1}, \frac{\alpha_*}{8\sqrt{2}L^2}\right\},
    \end{equation}
    then along each iteration of Alg~\ref{alg:lapd}, it has
    \begin{equation*}
        \mathrm{KL}(\tilde{p}_{k}\|p_*)\le e^{-\alpha_*\eta_k}\mathrm{KL}(\tilde{p}_{k-1}\|p_*) + 16\eta_k^2\mathrm{Tr}(\mH).
    \end{equation*}
\end{lemma}

\subsection{Discussion on specific examples}
\label{sec:diss}
In this section, we discuss some specific examples which can be coupled with Alg~\ref{alg:lapd} to demonstrate that LAPD is better than ULA in general and has a greater advantage in some cases. 
\paragraph{General target distributions and the choice of $g$.} 
From Table~\ref{tab:comp_old}, it can be observed that the convergence rate of LAPD can directly imply the convergence rate of ULA by taking $g(\vw) = 0$ and using the inequality $\mathrm{Tr}(\mH)\le \mathrm{Tr}(L^2\mI)=L^2d$. Moreover, since the parameter $m>0$ can be arbitrarily chosen, LAPD can be considered as a more general version of ULA and is able to achieve a faster convergence by properly tuning $m$. In particular, note that the convergence rate of LAPD is roughly $\mathcal O(\mathrm{Tr}(\mH_m)/\epsilon)$, where $\mH_m$ is a union upper bound of  $\nabla^2 f_m(\vw)(\nabla^2 f_m(\vw))^\top$, where $f_m(\vw)=U(\vw)-m\|\vw\|^2/2$. In contrast, the convergence rate of ULA is roughly $\tilde O(\mathrm{Tr}(\mH_0/\epsilon)$, which in the worse case could be $\tilde O\big(\mathrm{Tr}(\mH_m^{1/2}+m\mI)(\mH_m^{1/2}+m\mI)^\top\big)/\epsilon\big)$, thus can be much slower than the convergence rate of LAPD, especially when $\mathrm{Tr}(\mH^{1/2})$ is small.

Moreover, we can also extend the function $g(\vw) = m\|\vw\|^2/2$ to a more general quadratic function, e.g., $g(\vw)=\vw^\top \boldsymbol{\Sigma}\vw/2$, the additional tunable parameters in $\boldsymbol{\Sigma}$ provides higher flexibility to further improve the convergence rate.

\paragraph{Dimension-independent convergence for Gaussian mixtures.}
We further provide some example non-log-concave data distributions to demonstrate the efficiency of LAPD. We show that by properly choosing $m$, the quantity $ \mathrm{Tr}(\mH)$ is independent of the problem dimension $d$.
We may consider $p_*$ as a simple Gaussian mixture density function presented as follows $p_*(\vw) \propto \frac{1}{K}\sum_{i=1}^K \exp\left(-\frac{1}{2}\left\|\vw-\boldsymbol{\mu}_i\right\|^2\right)$, which means
\begin{equation*}
    U(\vw)= -\ln\left[\frac{1}{K}\sum_{i=1}^K e^{\boldsymbol{\mu}_i^\top \vw - \frac{1}{2}\left\|\boldsymbol{\mu}_i\right\|^2}\right]+\frac{1}{2}\left\|\vw\right\|^2 \coloneqq f(\vw)+g(\vw).
\end{equation*}
We should note that $p_*$ is not strongly log-concave but satisfies log-Sobolev inequality~\cite{chen2021dimension}.
If we define $v_i=\exp(\boldsymbol{\mu}_i^\top \vw - \frac{1}{2} \|\boldsymbol{\mu}_i\|^2)\in\R_+$, the Hessian of $f$ satisfies
\begin{equation*}
    \begin{aligned}
        \grad^2 f(\vw) = &\left(\sum_{i=1}^n v_i\right)^{-2}\cdot \left(\sum_{i=1}^n v_i\boldsymbol{\mu}_i \sum_{j=1}^n v_j\boldsymbol{\mu}_j^\top - \sum_{i=1}^n v_i\boldsymbol{\mu}_i\boldsymbol{\mu}_i^\top \sum_{j=1}^n v_j\right)\\
        = & -\frac{1}{2}\left(\sum_{i=1}^n v_i\right)^{-2}\cdot\sum_{1\le i\le j\le n} v_iv_j\left(\boldsymbol{\mu}_i-\boldsymbol{\mu}_j\right)\left(\boldsymbol{\mu}_i-\boldsymbol{\mu}_j\right)^\top.
    \end{aligned}
\end{equation*}
Then we can set $\mH^{1/2} = \sum_{1\le i\le j\le K}\left(\boldsymbol{\mu}_i-\boldsymbol{\mu}_j\right)\left(\boldsymbol{\mu}_i-\boldsymbol{\mu}_j\right)^\top$ that satisfies $0\preceq -\nabla^2f(\vw)\preceq \mH^{1/2}$.
\begin{equation*}
    \begin{aligned}
    \mathrm{Tr}(\mH) \le \mathrm{Tr}(\mH^{1/2})^2 = \Tr\left(\sum_{1\le i\le j\le K}\left(\boldsymbol{\mu}_i-\boldsymbol{\mu}_j\right)\left(\boldsymbol{\mu}_i-\boldsymbol{\mu}_j\right)^\top\right)^2 \le 16K^4R_{\mu}^4,
    \end{aligned}
\end{equation*}
where the means of the Gaussian mixture satisfy $\|\boldsymbol{\mu}_i\|\le R_{\mu}$.
This implies that the convergence rate of LAPD  only depends on the number of mixture components $K$ and the radius of means $R_{\mu}$, rather than the dimension number $d$ due to Theorem~\ref{thm:kl_con_fx}, thus can be regarded as dimension-independent.
\section{Conclusions and Future Work}
In this paper, we study the complexity of the (unadjusted) Langevin algorithm with prior diffusion, in terms of KL divergence. By carefully tracking the dynamics of our proposed algorithm and the associated discretization errors, we discover that under the Logarithmic Sobolev Inequality (LSI), the algorithm's convergence can be independent of the dimensionality, applicable for both fixed and variable step sizes. This finding aligns well with the analysis provided by \citet{vempala2019rapid}, which offers one of the most versatile frameworks for analyzing overdamped Langevin dynamics. This alignment prompts us to propose an intriguing avenue for future research: extending our analysis to underdamped Langevin dynamics to explore the possibility of dimension-independent convergence in higher-order Stochastic Differential Equations (SDEs).

Furthermore, given the typically higher dimension dependence in the convergence rates of unbiased samplers, such as the Metropolis Adjusted Langevin Algorithm (MALA) \citep{dwivedi2019log} and proximal samplers \citep{chen2020fast}, which rely on a restricted Gaussian oracle, our findings suggest a promising direction. By integrating our approach (Alg~\ref{alg:lapd}) with fast, unbiased sampling algorithms, we anticipate the development of novel unbiased sampling methods that boast even more rapid convergence rates. Additionally, we propose to broaden our analysis to stochastic settings, utilizing stochastic gradients for Langevin algorithm updates. We believe that the dimension dependency established in \citet{raginsky2017non,xu2018global,zou2021faster} can be also improved.

\newpage
\bibliographystyle{apalike}
\bibliography{0_contents/ref}  %%% Uncomment this line and comment out the ``thebibliography'' section below to use the external .bib file (using bibtex) .

\begin{thebibliography}{}

\bibitem[Altschuler and Chewi, 2023]{altschuler2023faster}
Altschuler, J.~M. and Chewi, S. (2023).
\newblock Faster high-accuracy log-concave sampling via algorithmic warm
  starts.
\newblock {\em arXiv preprint arXiv:2302.10249}.

\bibitem[Ambrosio et~al., 2005]{ambrosio2005gradient}
Ambrosio, L., Gigli, N., and Savar{\'e}, G. (2005).
\newblock {\em Gradient flows: in metric spaces and in the space of probability
  measures}.
\newblock Springer Science \& Business Media.

\bibitem[Bakry and {\'E}mery, 1985]{bakry1985diffusions}
Bakry, D. and {\'E}mery, M. (1985).
\newblock Diffusions hypercontractives.
\newblock In {\em Seminaire de probabilit{\'e}s XIX 1983/84}, pages 177--206.
  Springer.

\bibitem[Bakry et~al., 2014]{bakry2014analysis}
Bakry, D., Gentil, I., Ledoux, M., et~al. (2014).
\newblock {\em Analysis and geometry of Markov diffusion operators}, volume
  103.
\newblock Springer.

\bibitem[Beck and Teboulle, 2009]{beck2009A}
Beck, A. and Teboulle, M. (2009).
\newblock A fast iterative shrinkage-thresholding algorithm for linear inverse
  problems.
\newblock {\em Siam J Imaging Sciences}, 2(1):183--202.

\bibitem[Bellman, 1968]{bellman1968some}
Bellman, R. (1968).
\newblock Some inequalities for the square root of a positive definite matrix.
\newblock {\em Linear Algebra and its applications}, 1(3):321--324.

\bibitem[Bishop and Nasrabadi, 2006]{bishop2006pattern}
Bishop, C.~M. and Nasrabadi, N.~M. (2006).
\newblock {\em Pattern recognition and machine learning}, volume~4.
\newblock Springer.

\bibitem[Brooks et~al., 2011]{brooks2011handbook}
Brooks, S., Gelman, A., Jones, G., and Meng, X.-L. (2011).
\newblock {\em Handbook of markov chain monte carlo}.
\newblock CRC press.

\bibitem[Chen et~al., 2021]{chen2021dimension}
Chen, H.-B., Chewi, S., and Niles-Weed, J. (2021).
\newblock Dimension-free log-sobolev inequalities for mixture distributions.
\newblock {\em Journal of Functional Analysis}, 281(11):109236.

\bibitem[Chen et~al., 2020]{chen2020fast}
Chen, Y., Dwivedi, R., Wainwright, M.~J., and Yu, B. (2020).
\newblock Fast mixing of metropolized hamiltonian monte carlo: Benefits of
  multi-step gradients.
\newblock {\em The Journal of Machine Learning Research}, 21(1):3647--3717.

\bibitem[Cheng and Bartlett, 2018]{cheng2018convergence}
Cheng, X. and Bartlett, P. (2018).
\newblock Convergence of langevin mcmc in kl-divergence.
\newblock In {\em Algorithmic Learning Theory}, pages 186--211. PMLR.

\bibitem[Dalalyan, 2017a]{dalalyan2017further}
Dalalyan, A. (2017a).
\newblock Further and stronger analogy between sampling and optimization:
  Langevin monte carlo and gradient descent.
\newblock In {\em Conference on Learning Theory}, pages 678--689. PMLR.

\bibitem[Dalalyan, 2017b]{dalalyan2017theoretical}
Dalalyan, A.~S. (2017b).
\newblock Theoretical guarantees for approximate sampling from smooth and
  log-concave densities.
\newblock {\em Journal of the Royal Statistical Society: Series B (Statistical
  Methodology)}, 79(3):651--676.

\bibitem[Dalalyan and Karagulyan, 2019]{dalalyan2019user}
Dalalyan, A.~S. and Karagulyan, A. (2019).
\newblock User-friendly guarantees for the langevin monte carlo with inaccurate
  gradient.
\newblock {\em Stochastic Processes and their Applications},
  129(12):5278--5311.

\bibitem[Durmus et~al., 2019]{durmus2019analysis}
Durmus, A., Majewski, S., and Miasojedow, B. (2019).
\newblock Analysis of langevin monte carlo via convex optimization.
\newblock {\em The Journal of Machine Learning Research}, 20(1):2666--2711.

\bibitem[Durmus and Moulines, 2019]{durmus2019high}
Durmus, A. and Moulines, E. (2019).
\newblock High-dimensional bayesian inference via the unadjusted langevin
  algorithm.
\newblock {\em Bernoulli}, 25(4A):2854--2882.

\bibitem[Dwivedi et~al., 2018]{dwivedi2018log}
Dwivedi, R., Chen, Y., Wainwright, M.~J., and Yu, B. (2018).
\newblock Log-concave sampling: Metropolis-hastings algorithms are fast!
\newblock In {\em Conference on learning theory}, pages 793--797. PMLR.

\bibitem[Dwivedi et~al., 2019]{dwivedi2019log}
Dwivedi, R., Chen, Y., Wainwright, M.~J., and Yu, B. (2019).
\newblock Log-concave sampling: Metropolis-hastings algorithms are fast.
\newblock {\em Journal of Machine Learning Research}, 20(183):1--42.

\bibitem[Freund et~al., 2022]{freund2022convergence}
Freund, Y., Ma, Y.-A., and Zhang, T. (2022).
\newblock When is the convergence time of langevin algorithms dimension
  independent? a composite optimization viewpoint.
\newblock {\em Journal of Machine Learning Research}, 23(214):1--32.

\bibitem[Gilks et~al., 1995]{gilks1995markov}
Gilks, W.~R., Richardson, S., and Spiegelhalter, D. (1995).
\newblock {\em Markov chain Monte Carlo in practice}.
\newblock CRC press.

\bibitem[Huang et~al., 2024a]{huang2024reverse}
Huang, X., Dong, H., HAO, Y., Ma, Y., and Zhang, T. (2024a).
\newblock Reverse diffusion monte carlo.
\newblock In {\em International Conference on Learning Representations (ICLR)}.

\bibitem[Huang et~al., 2024b]{huang2024faster}
Huang, X., Zou, D., Dong, H., Ma, Y., and Zhang, T. (2024b).
\newblock Faster sampling without isoperimetry via diffusion-based monte carlo.
\newblock {\em arXiv preprint arXiv:2401.06325}.

\bibitem[Jordan et~al., 1999]{jordan1999introduction}
Jordan, M.~I., Ghahramani, Z., Jaakkola, T.~S., and Saul, L.~K. (1999).
\newblock An introduction to variational methods for graphical models.
\newblock {\em Machine learning}, 37(2):183--233.

\bibitem[Jordan et~al., 1998]{jordan1998variational}
Jordan, R., Kinderlehrer, D., and Otto, F. (1998).
\newblock The variational formulation of the fokker--planck equation.
\newblock {\em SIAM journal on mathematical analysis}, 29(1):1--17.

\bibitem[Liu et~al., 2024]{liu2024double}
Liu, Y., Fang, C., and Zhang, T. (2024).
\newblock Double randomized underdamped langevin with dimension-independent
  convergence guarantee.
\newblock {\em Advances in Neural Information Processing Systems}, 36.

\bibitem[Ma et~al., 2019a]{ma2019there}
Ma, Y.-A., Chatterji, N., Cheng, X., Flammarion, N., Bartlett, P., and Jordan,
  M.~I. (2019a).
\newblock Is there an analog of nesterov acceleration for mcmc?
\newblock {\em arXiv preprint arXiv:1902.00996}.

\bibitem[Ma et~al., 2019b]{ma2019sampling}
Ma, Y.-A., Chen, Y., Jin, C., Flammarion, N., and Jordan, M.~I. (2019b).
\newblock Sampling can be faster than optimization.
\newblock {\em Proceedings of the National Academy of Sciences},
  116(42):20881--20885.

\bibitem[Neal, 1993]{neal1993probabilistic}
Neal, R.~M. (1993).
\newblock {\em Probabilistic inference using Markov chain Monte Carlo methods}.
\newblock Department of Computer Science, University of Toronto Toronto, ON,
  Canada.

\bibitem[Nesterov et~al., 2018]{nesterov2018lectures}
Nesterov, Y. et~al. (2018).
\newblock {\em Lectures on convex optimization}, volume 137.
\newblock Springer.

\bibitem[Oksendal, 2013]{oksendal2013stochastic}
Oksendal, B. (2013).
\newblock {\em Stochastic differential equations: an introduction with
  applications}.
\newblock Springer Science \& Business Media.

\bibitem[Raginsky et~al., 2017]{raginsky2017non}
Raginsky, M., Rakhlin, A., and Telgarsky, M. (2017).
\newblock Non-convex learning via stochastic gradient langevin dynamics: a
  nonasymptotic analysis.
\newblock In {\em Conference on Learning Theory}, pages 1674--1703. PMLR.

\bibitem[Risken, 1996]{risken1996fokker}
Risken, H. (1996).
\newblock Fokker-planck equation.
\newblock In {\em The Fokker-Planck Equation}, pages 63--95. Springer.

\bibitem[Robert et~al., 1999]{robert1999monte}
Robert, C.~P., Casella, G., and Casella, G. (1999).
\newblock {\em Monte Carlo statistical methods}, volume~2.
\newblock Springer.

\bibitem[Roberts and Stramer, 2002]{roberts2002langevin}
Roberts, G.~O. and Stramer, O. (2002).
\newblock Langevin diffusions and metropolis-hastings algorithms.
\newblock {\em Methodology and computing in applied probability},
  4(4):337--357.

\bibitem[Rossky et~al., 1978]{rossky1978brownian}
Rossky, P.~J., Doll, J.~D., and Friedman, H.~L. (1978).
\newblock Brownian dynamics as smart monte carlo simulation.
\newblock {\em The Journal of Chemical Physics}, 69(10):4628--4633.

\bibitem[S{\"a}rkk{\"a} and Solin, 2019]{sarkka2019applied}
S{\"a}rkk{\"a}, S. and Solin, A. (2019).
\newblock {\em Applied stochastic differential equations}, volume~10.
\newblock Cambridge University Press.

\bibitem[Shen and Lee, 2019]{shen2019randomized}
Shen, R. and Lee, Y.~T. (2019).
\newblock The randomized midpoint method for log-concave sampling.
\newblock {\em Advances in Neural Information Processing Systems}, 32.

\bibitem[Teh et~al., 2016]{teh2016consistency}
Teh, Y.~W., Thiery, A.~H., and Vollmer, S.~J. (2016).
\newblock Consistency and fluctuations for stochastic gradient langevin
  dynamics.
\newblock {\em Journal of Machine Learning Research}, 17.

\bibitem[Vempala and Wibisono, 2019]{vempala2019rapid}
Vempala, S. and Wibisono, A. (2019).
\newblock Rapid convergence of the unadjusted langevin algorithm: Isoperimetry
  suffices.
\newblock {\em Advances in neural information processing systems}, 32.

\bibitem[Welling and Teh, 2011]{welling2011bayesian}
Welling, M. and Teh, Y.~W. (2011).
\newblock Bayesian learning via stochastic gradient langevin dynamics.
\newblock In {\em Proceedings of the 28th international conference on machine
  learning (ICML-11)}, pages 681--688. Citeseer.

\bibitem[Wibisono, 2018]{wibisono2018sampling}
Wibisono, A. (2018).
\newblock Sampling as optimization in the space of measures: The langevin
  dynamics as a composite optimization problem.
\newblock In {\em Conference on Learning Theory}, pages 2093--3027. PMLR.

\bibitem[Wibisono, 2019]{wibisono2019proximal}
Wibisono, A. (2019).
\newblock Proximal langevin algorithm: Rapid convergence under isoperimetry.
\newblock {\em arXiv preprint arXiv:1911.01469}.

\bibitem[Xifara et~al., 2014]{xifara2014langevin}
Xifara, T., Sherlock, C., Livingstone, S., Byrne, S., and Girolami, M. (2014).
\newblock Langevin diffusions and the metropolis-adjusted langevin algorithm.
\newblock {\em Statistics \& Probability Letters}, 91:14--19.

\bibitem[Xu et~al., 2018]{xu2018global}
Xu, P., Chen, J., Zou, D., and Gu, Q. (2018).
\newblock Global convergence of {L}angevin dynamics based algorithms for
  nonconvex optimization.
\newblock In {\em Advances in Neural Information Processing Systems}, pages
  3126--3137.

\bibitem[Zou et~al., 2021]{zou2021faster}
Zou, D., Xu, P., and Gu, Q. (2021).
\newblock Faster convergence of stochastic gradient langevin dynamics for
  non-log-concave sampling.
\newblock In {\em Uncertainty in Artificial Intelligence}, pages 1152--1162.
  PMLR.

\end{thebibliography}





% \end{thebibliography}
\newpage
\appendix
\section{Proof of Lemma~\ref{lem:con_equivalent_sde}}
\label{sec:implictsde}
Consider the posterior distribution with the following density function:
\begin{equation}
    \label{def:tar_dis}
    p_*(\vw) = \exp \left(-U(\vw)\right).
\end{equation}
Such density function can be reformulated as
\begin{equation}
    \label{def:tar_dis_com}
    p_*(\vw) = \exp \left(-(U(\vw)-g(\vw)) - g(\vw)\right)\coloneqq \exp \left(-f(\vw) - g(\vw)\right).
\end{equation}
For simplicity, we suppose $g(\vw)=(m/2)\left\|\vw\right\|^2$.
With the same algorithm as~\citet{freund2022convergence}, There are some different notations. 
Each iteration of the algorithm is divided into two stages. 
In this first stage, it has
\begin{equation}
    \label{eq:iter_stage_2}
    \rvw_{k} = \tilde{\rvw}_{k-1} - \tilde{\eta}_{k-1}\grad f(\tilde{\rvw}_{k-1}).
\end{equation}
Then, for the second stage, there is
\begin{equation}
    \label{eq:iter_stage_1}
    \tilde{\rvw}_{k}\coloneqq \tilde{\rvw}(\eta_k) = \rvw_{k} -\int_0^{\eta_k} \grad g(\tilde{\rvw}(s))\der s + \sqrt{2}\int_0^{\eta_k}\der \mB_t,
\end{equation}
where $\mB_t$ denotes the standard Brownian motion. 
In the following, we only consider one iteration and abbreviate $\tilde{\eta}_{k-1}$ and $\eta_k$ as $\tilde{\eta}$ and $\eta$.
Due to the quadratic form of $g(\vw)$, stage 2 is a standard Ornstein–Uhlenbeck (OU) process, which implies, the transition density is given as follows
\begin{equation*}
    \tilde{p}(\vw, t | \vw^\prime, t^\prime) = \left(\frac{2\pi (1-e^{-2m(t-t^\prime)})}{m}\right)^{-d/2}\cdot \exp \left[-\frac{m}{2}\cdot \frac{\left\|\vw-\vw^\prime e^{-m(t-t^\prime)}\right\|^2}{1-e^{-2m(t-t^\prime)}}\right].
\end{equation*}
Hence, conditioning on $\rvw_k = \vw_0^\prime$, we have the probability density of $\tilde{\rvw}_t$ is given as
\begin{equation*}
    \tilde{p}(\vw, \eta | \vw_0^\prime, 0) = \left(\frac{2\pi (1-e^{-2m\eta})}{m}\right)^{-d/2}\cdot \exp \left[-\frac{m}{2}\cdot \frac{\left\|\vw-\vw_0^\prime e^{-m\eta}\right\|^2}{1-e^{-2m\eta}}\right].
\end{equation*}
With the change of variable, i.e., Lemma~\ref{lem:changeofvar}, we can deduce the probability density of $\tilde{\rvw}_t$ conditioning on $\tilde{\rvw}_{k-1}=\vw_0$ as
\begin{equation*}
     p(\vw, \eta | \vw_0, 0) = \left(\frac{2\pi (1-e^{-2m\eta})}{m}\right)^{-d/2}\cdot \exp \left[-\frac{m}{2}\cdot \frac{\left\|\vw-\left(\vw_0 - \tilde{\eta} \grad f(\vw_0)\right) e^{-m\eta}\right\|^2}{1-e^{-2m\eta}}\right].
\end{equation*}
Hence, conditioning on $\tilde{\rvw}_{k-1}=\vw_0$, we expect some transition density $p_{\vw_0}(\vw, t|\vw^\prime, t^\prime)$, solving some equivalent It\^o SDE, to satisfy $p_{\vw_0}(\vw, \eta|\vw_0, 0)=p(\vw, \eta | \vw_0, 0)$ for any $\vw\in\R^d$. 
Hence, we consider the following transition density
\begin{equation}
    \small
    \label{def:impsde_kernel}
    \begin{aligned}
        &p_{\vw_0}(\vw, t | \vw^\prime, t^\prime)\\
        &=  \left(\frac{2\pi (1-e^{-2m(t-t^\prime)})}{m}\right)^{-d/2}\cdot \exp \left[-\frac{m}{2}\cdot \frac{\left\|\vw-\left(\vw^\prime - \left(e^{m(t-t^\prime)}-1\right)\cdot \frac{\tilde{\eta}}{e^{m\eta}-1}\grad f(\vw_0)\right) e^{-m(t-t^\prime)}\right\|^2}{1-e^{-2m(t-t^\prime)}}\right]\\
        &= \left(\frac{2\pi (1-e^{-2m(t-t^\prime)})}{m}\right)^{-d/2}\cdot \exp \left[-\frac{m}{2}\cdot \frac{\left\|\vw-\vw^\prime e^{-m(t-t^\prime)}+ \left(1-e^{-m(t-t^\prime)}\right)\cdot \left(\frac{\tilde{\eta}}{e^{m\eta} -1}\right)\grad f(\vw_0)\right\|^2}{1-e^{-2m(t-t^\prime)}}\right],
    \end{aligned}
\end{equation}
and try to obtain the equivalent SDE.
For clarity, we leave the calculation of the partial derivatives in Appendix~\ref{sec:par_cal}.
Conditioning on $\tilde{\rvw}_{k-1}=\vw_0$, we have the following equation
\begin{equation}
    \label{def:kbe_equ}
    -\frac{\partial p_{\vw_0}(\vw, t| \vw^\prime, t^\prime)}{\partial t^\prime} = \left(-m\vw^\prime - \frac{m\tilde{\eta}}{e^{m\eta}-1} \grad f(\vw_0)\right)\cdot \frac{\partial p_{\vw_0}(\vw, t|\vw^\prime, t^\prime)}{\partial \vw^\prime} + 1\cdot \sum_{i=1}^d \frac{\partial^2 p(\vw,t | \vw^\prime, t^\prime ,\vw_0)}{\left(\partial \vw_i\right)^2}.
\end{equation}
Suppose there is an It\^o SDE 
\begin{equation*}
    \der\hat{\rvw}_t = -\vmu(\hat{\rvw}_t)\der t + \vsigma(\hat{\rvw}_t)\der \mB_t.
\end{equation*}
Its the infinitesimal operator is given as that in Eq.~\ref{def:ifosde}, i.e.,
\begin{equation}
    \label{def:ifo_impsde}
    \gL \phi(\vw) =\sum_{i} \frac{\partial \phi(\vw)}{\partial \vw_i} \vmu_i(\vw) + \frac{1}{2}\sum_{i,j}\left[\frac{\partial^2 \phi(\vw)}{\partial \vw_i \partial \vw_j}\right]\cdot \left[\vsigma(\vw)\vsigma(\vw)^\top\right]_{ij}.
\end{equation}
If we choose
\begin{equation*}
    \vmu(\vw^\prime)=-m\vw^\prime - \frac{m\tilde{\eta}}{e^{m\eta}-1} \grad f(\vw_0)\quad \mathrm{and}\quad \vsigma(\vw^\prime)=\sqrt{2}\mI,
\end{equation*}
Eq.~\ref{def:kbe_equ} denotes that then $p_{\vw_0}$ solves the following backward Kolmogorov equation
\begin{equation*}
        -\frac{\partial p_{\vw_0}(\vw_t, t|\vw_s, s)}{\partial s}=\gL p_{\vw_0}(\vw_t, t|\vw_s, s),\quad p_{\vw_0}(\vw^\prime,s|\vw,s)=\delta(\vw^\prime - \vw).
\end{equation*}
According to Lemma~\ref{lem:kbe}, we can consider $p_{\vw_0}$ as the transition density of $\left\{\hat{\rvw}_t\right\}_{t\ge 0}$ obtained by the following SDE
\begin{equation*}
    \der \hat{\rvw}_t = -\left(m\hat{\rvw}_t+\frac{m\tilde{\eta}_{k-1}}{e^{m\eta_k}-1}\grad f(\vw_0)\right)\der t + \sqrt{2}\der \mB_t.
\end{equation*}
It means, conditioning on $\tilde{\rvw}_{t-1}=\vw_0$, $\hat{\rvw}_{\eta_k}$ shares the same distribution as $\tilde{\rvw}_k$ because 
\begin{equation*}
    p_{\vw_0}(\vw, \eta|\vw_0, 0)=p(\vw, \eta | \vw_0, 0).
\end{equation*}
Hence, the proof is completed.

\section{Convergence of KL divergence under LSI}
\label{appsec:main_results}
With the same setting as Appendix~\ref{sec:implictsde}, we assume $U(\vw)$ in Eq.~\ref{def:tar_dis} satisfies the following assumptions. 
\begin{enumerate}[label=\textbf{[{A}{\arabic*}]}]
    \item \label{ass:app_pri_g} The function $g$ is $m$-strongly convex and can be explicitly integrated. For simplicity, we choose $g(\vw)=(m/2)\left\|\vw\right\|^2$ in following sections.
    \item \label{ass:post_f} $L_U$-smoothness of negative log-likelihood, i.e., for any $\vw\in\R^d$, we suppose
    \begin{equation*}
        \left\|\nabla^2 f(\vw)\right\|\preceq L\mI.
    \end{equation*}
    \item \label{ass:post_f2} Similar to the assumption in~\cite{freund2022convergence}, for any $\vw\in\R^d$, we suppose
    \begin{equation*}
        \vzero\preceq \nabla^2 f(\vw) (\grad^2 f(\vw))^\top\preceq \mH^{\frac{1}{2}}\left(\mH^{\frac{1}{2}}\right)^\top = \mH \preceq L^2\mI,\quad \mathrm{where}\quad \vzero \preceq \mH^{\frac{1}{2}}.
    \end{equation*}
    \item \label{ass:post_lsi} The posterior distribution, $\mu_* \propto \exp(-U)$, satisfies log-Sobolev inequality (LSI). It means
    \begin{equation*}
        \mathbb{E}_{\mu_*}[q^2\log q^2]-\mathbb{E}_{\mu_*}[q^2]\log \mathbb{E}_{\mu_*}[q^2]\le \frac{2}{\alpha_*}\mathbb{E}_{\nu_f}\left[\left\|\grad q\right\|^2\right]
    \end{equation*}
    for any smooth function $q\colon \R^d\rightarrow\R$ where 
    \begin{equation*}
        \der \mu_*(\vw) = \exp \left(-\left(f(\vw) + g(\vw)\right)\right) \der \vw = \exp (-U(\vw))\der\vw.
    \end{equation*}
    
\end{enumerate}
For any suitable probability density function $\tilde{p}_{k-1}$ of $\tilde{\rvw}_{k-1}$, we set
\begin{equation*}
    \hat{p}_t(\vw)\coloneqq \int_{\vw_0} p_{\vw_0}(\vw, t|\vw_0, 0)\cdot \tilde{p}_{k-1}(\vw_0)\der \vw_0,
\end{equation*}
where $p_{\vw_0}$ is defined in Eq.~\ref{def:impsde_kernel}, and
\begin{equation*}
    \begin{aligned}
        \hat{p}_{t|0}(\vw|\vw_0) \coloneqq & p_{\vw_0}(\vw, t|\vw_0, 0)\\
        \hat{p}_{0t}(\vw_0, \vw) \coloneqq & \tilde{p}_{k-1}(\vw_0)\cdot p_{\vw_0}(\vw, t|\vw_0, 0)
    \end{aligned}
\end{equation*}
for abbreviation.
According to Appendix~\ref{sec:implictsde}, we have $\hat{p}_{\eta_k}(\vw)=\tilde{p}_k(\vw)$ for any $\vw\in\R^d$ where $\tilde{p}_k$ denotes the density function of $\tilde{\rvw}_k$.
It also implies
\begin{equation*}
    \mathrm{KL}\left(\hat{p}_{\eta_k} \| p^*\right) = \mathrm{KL}\left(\tilde{p}_{k}\| p^*\right)\quad \mathrm{and}\quad \mathrm{KL}\left(\hat{p}_0 \| p^*\right) = \mathrm{KL}\left(\tilde{p}_{k-1}\| p^*\right).
\end{equation*}

\subsection{Proof of Lemma~\ref{lem:con_kl_contraction}}
Conditioning on $\vw_0$, the Kolmogorov forward equation (Fokker-Planck equation), i.e., Lemma~\ref{lem:fpe},  shows 
\begin{equation*}
    \frac{\partial \hat{p}_{t|0}(\vw|\vw_0)}{\partial t} = \gL^*\hat{p}_{t|0}(\vw|\vw_0) = \grad \cdot \left(\hat{p}_{t|0}(\vw|\vw_0)\left(m\vw + \frac{m\tilde{\eta}_{k-1}}{e^{m\eta_k}-1}\grad f(\vw_0)\right)\right)+\Delta \hat{p}_{t|0}\left(\vw|\vw_0\right),
\end{equation*}
where $\gL^*$ is adjoint operator of $\gL$ given in Eq.~\ref{def:ifo_impsde}.
In this condition, we set
\begin{equation}
    \label{eq:eta_connection}
    \left(m\tilde{\eta}_{k-1}\right)/\left(e^{m\eta_k}-1\right)=1
\end{equation}
in the following.
It means the evolution of $p_t$ can be derived as follows.
\begin{equation}
    \begin{aligned}
        \frac{\partial \hat{p}_t(\vw)}{\partial t}=& \int \frac{\partial \hat{p}_{t|0}(\vw|\vw_0)}{\partial t} p_0(\vw_0)\der\vw_0\\
        = & \int \left[\grad \cdot \left(\hat{p}_{t|0}(\vw|\vw_0)\left(m\vw + \grad f(\vw_0)\right)\right)+\Delta \hat{p}_{t|0}\left(\vw|\vw_0\right)\right]p_0(\vw_0)\der \vw_0\\
        = & \int \grad \cdot \left(\hat{p}_{0t}(\vw_0, \vw)\left(m\vw + \grad f(\vw_0)\right)\right) \der \vw_0 + \int \Delta \hat{p}_{0t}(\vw_0, \vw) \der \vw_0\\
        =& \grad \cdot \left(\hat{p}_t(\vw)\int \hat{p}_{0|t}\left(\vw_0|\vw\right)\left(m\vw + \grad f(\vw_0)\right)\right) + \Delta \hat{p}_t(\vw).
    \end{aligned}
\end{equation}
Hence, the dynamic of KL divergence  can be obtained as
\begin{equation}
    \label{eq:kld_imp_sde}
    \begin{aligned}
        &\frac{\der \mathrm{KL}(\hat{p}_t\|p_*)}{\der t}\notag\\
        &=  \frac{\der}{\der t}\int \hat{p}_t(\vw) \log \frac{\hat{p}_t(\vw)}{p_*(\vw)}\der \vw = \int \frac{\partial \hat{p}_t(\vw)}{\partial t}\log \frac{\hat{p}_t(\vw)}{p_*(\vw)}\der \vw\\
        &=  \int \left[\grad \cdot \left(\hat{p}_t(\vw)\int \hat{p}_{0|t}\left(\vw_0|\vw\right)\left(m\vw + \grad f(\vw_0)\right)\right) + \Delta \hat{p}_t(\vw)\right] \log \frac{\hat{p}_t(\vw)}{p_*(\vw)}\der \vw\\
        &=  \int \grad \cdot \left(\hat{p}_t(\vw)\cdot \left(\int \hat{p}_{0|t}(\vw_0|\vw)\left(m\vw+\grad f(\vw_0)\right)\der \vw_0+ \grad \log \hat{p}_t(\vw)\right)\right) \log \frac{\hat{p}_t(\vw)}{p_*(\vw)}\der \vw\\
        &=  -\int \hat{p}_t(\vw)\cdot \left[\grad \log \frac{\hat{p}_t(\vw)}{p_*(\vw)}+ \int \hat{p}_{0|t}(\vw_0|\vw)\left(m\vw+\grad f(\vw_0)\right)\der \vw_0 - \left(m\vw + \grad f(\vw)\right)\right]\\
        &\qquad \grad \log \frac{\hat{p}_t(\vw)}{p_*(\vw)}\der\vw\\
        &=  -\int \hat{p}_t(\vw) \left\|\grad \log \frac{\hat{p}_t(\vw)}{p_*(\vw)}\right\|^2 \der \vw - \int \hat{p}_{0t}(\vw_0, \vw) \left(\grad f(\vw_0)-\grad f(\vw)\right)^\top \grad \log \frac{\hat{p}_t(\vw)}{p_*(\vw)}\der (\vw_0,\vw)\\
        &\le  -\frac{3}{4}\int \hat{p}_t(\vw) \left\|\grad \log \frac{\hat{p}_t(\vw)}{p_*(\vw)}\right\|^2\der \vw + \underbrace{\int \hat{p}_{0t}\left(\vw_0, \vw\right)\left\|\grad f(\vw_0)-\grad f(\vw)\right\|^2 \der (\vw_0, \vw)}_{\mathrm{term\ 1}}.
    \end{aligned}
\end{equation}
In the following, we start to upper bound $\mathrm{term\ 1}$ of Eq.~\ref{eq:kld_imp_sde}.
For abbreviation, we set
\begin{equation*}
    \overline{\vw}(t, \vw_0)=e^{-mt}\vw_0 - \frac{1-e^{-mt}}{m} \grad f(\vw_0),
\end{equation*}
and have
\begin{equation}
    \label{ineq:upb_term1}
    \begin{aligned}
        \mathrm{term\ 1} = &\int \hat{p}_0(\vw_0) \int \hat{p}_{t|0}(\vw|\vw_0)\left\|\grad f(\vw)-\grad f(\vw_0)\right\|^2\der \vw \der\vw_0\\
        = & \int \hat{p}_0(\vw_0) \int \hat{p}_{t|0}(\vw|\vw_0)\left\|\grad f(\vw)-\grad f(\overline{\vw}(t,\vw_0))+ \grad f(\overline{\vw}(t,\vw_0)) -\grad f(\vw_0)\right\|^2\der \vw \der\vw_0\\
        \le & 2\underbrace{\int \hat{p}_0(\vw_0) \int \hat{p}_{t|0}(\vw|\vw_0)\left\|\grad f(\vw)-\grad f(\overline{\vw}(t,\vw_0))\right\|^2\der \vw \der\vw_0}_{\mathrm{term\ 1.1}}\\
        & + 2\underbrace{\int \hat{p}_0(\vw_0) \left\|\grad f(\overline{\vw}(t,\vw_0)) -\grad f(\vw_0)\right\|^2 \der\vw_0}_{\mathrm{term\ 1.2}}.
    \end{aligned}
\end{equation}
We first consider $\mathrm{term\ 1.1}$.
Following from Lemma~\ref{lem:granorm_bound}, for any $\vw, \overline{\vw}(t,\vw_0)$ we have
\begin{equation*}
    \left\|\grad f(\vw) - \grad f(\overline{\vw}(t,\vw_0))\right\|^2 \le \left(\vw-\overline{\vw}(t,\vw_0)\right)^\top\mH\left(\vw - \overline{\vw}(t,\vw_0)\right) = \left\|\mH^{\frac{1}{2}}\left(\vw - \overline{\vw}(t,\vw_0)\right)\right\|^2
\end{equation*}

According to the transition kernel conditioning on $\vw_0$ shown in Eq.~\ref{def:impsde_kernel}, we have
\begin{equation}
    \label{ineq:upb_term11}
    \small
    \begin{aligned}
        \mathrm{term\ 1.1} = & \int \hat{p}_{0}(\vw_0)\int \left(\frac{2\pi (1-e^{-2mt})}{m}\right)^{-d/2}\cdot \exp \left(-\frac{m}{2}\cdot \frac{\left\|\vw-\overline{\vw}(t,\vw_0)\right\|^2}{1-e^{-2mt}}\right)\cdot \left\|\grad f(\vw)-\grad f(\overline{\vw}(t,\vw_0))\right\|^2 \der\vw\der\vw_0\\
        \le & \int \hat{p}_{0}(\vw_0)\int \left(\frac{2\pi (1-e^{-2mt})}{m}\right)^{-d/2}\cdot \exp \left(-\frac{m}{2}\cdot \frac{\left\|\vw-\overline{\vw}(t,\vw_0)\right\|^2}{1-e^{-2mt}}\right)\cdot \left\|\mH^{\frac{1}{2}}\left(\vw - \overline{\vw}(t,\vw_0)\right)\right\|^2 \der\vw\der\vw_0\\
        =& \int \hat{p}_{0}(\vw_0) \sum_{i=1}^d\mathrm{var}_{\rvw\sim \hat{p}_{t|0}}\left[[\mH^{\frac{1}{2}}\rvw]_i\right] \der\vw_0 = \int \hat{p}_{0}(\vw_0)\mathrm{Tr}\left(\mathrm{cov}\left[\mH^{\frac{1}{2}}\rvw\right] \right) \der \vw_0\\
        = & \frac{1-e^{-2mt}}{m}\int \hat{p}_{0}(\vw_0)\mathrm{Tr}(\mH) \der\vw_0 = \frac{1-e^{-2mt}}{m}\cdot\mathrm{Tr}(\mH),
    \end{aligned}
\end{equation}
where the fourth equation follows from the definition of $\hat{p}_{t|0}$ and Lemma~\ref{lem:cov_bound}.

Similar to $\mathrm{term\ 1.1}$, following from Lemma~\ref{lem:granorm_bound}, $\mathrm{term\ 1.2}$ satisfies
\begin{equation*}
    \label{ineq:upb_term12_inner_part_rev}
    \begin{aligned}
        \left\|\grad f(\overline{\vw}(t,\vw_0))- \grad f(\vw_0)\right\|^2\le & \left(\overline{\vw}(t,\vw_0) - \vw_0\right)^\top \mH \left(\overline{\vw}(t,\vw_0) - \vw_0\right)\\
        = & \left(\frac{1-e^{-mt}}{m}\right)^2\cdot \left(m\vw_0 + \grad f(\vw_0)\right)^\top\mH\left(m\vw_0 + \grad f(\vw_0)\right)\\
        = & \left(\frac{1-e^{-mt}}{m}\right)^2\cdot \grad U(\vw_0)^\top \mH \grad U(\vw_0)
    \end{aligned}
\end{equation*}
Then, suppose $\rvw_0\sim \hat{p}_0$ and $\rvw^* \sim p_*$ with an optimal coupling $(\rvw^*, \rvw_0)\sim \gamma$, and we obtain
\begin{equation}
    \label{ineq:upb_term12_rev}
    \begin{aligned}
        \mathrm{term\ 1.2} = & \int \hat{p}_0(\vw_0) \left\|\grad f(\overline{\vw}(t,\vw_0)) -\grad f(\vw_0)\right\|^2 \der\vw_0\\
        \le &\left(\frac{1-e^{-mt}}{m}\right)^2\cdot \int \gamma(\vw^*, \vw_0)\left\|\mH^{\frac{1}{2}}\left(\grad U(\vw_0)-\grad U(\vw^*)\right)+\mH^{\frac{1}{2}}\grad U(\vw^*)\right\|^2\der(\vw^*, \vw_0)\\
        \le & \left(\frac{1-e^{-mt}}{m}\right)^2\cdot \int 2L^2(L+m)^2\cdot \gamma(\vw^*, \vw_0)\left\|\vw^* -\vw_0\right\|^2 \der (\vw^*, \vw_0)\\
        &+\left(\frac{1-e^{-mt}}{m}\right)^2\cdot \int 2 p_*(\vw^*)\left\|\mH^{\frac{1}{2}}\grad U(\vw^*)\right\|^2 \der \vw^*\\
        = & 2L^2(L+m)^2\cdot \left(\frac{1-e^{-mt}}{m}\right)^2 \cdot \frac{2}{\alpha_*}\mathrm{KL}\left(\hat{p}_0\|p_*\right) + 2\left(\frac{1-e^{-mt}}{m}\right)^2 \cdot \int p_*(\vw^*)\left\|\mH^{\frac{1}{2}}\grad U(\vw^*)\right\|^2 \der \vw^*
    \end{aligned}
\end{equation}
For the last term, it satisfies
\begin{equation*}
    \begin{aligned}
        &\int p_*(\vw^*)\left\|\mH^{\frac{1}{2}}\grad U(\vw^*)\right\|^2 \der \vw^*= \int \grad p_*^\top(\vw^*)\mH \grad\ln p_*(\vw^*)\der\vw^*\\
        = & \int p_*(\vw^*) \Tr(\mH \left(-\grad^2\ln p_*(\vw^*)\right))\der \vw^*\\
        = & \int p_*(\vw^*) \Tr(\mH \left(\grad^2 f(\vw^*)+m\mI\right))\der \vw^* \le \Tr(\mH^{3/2})+m\Tr(\mH),
    \end{aligned}
\end{equation*}
%\dz{In fact, the above bound can be tightened as $\Tr(\mH^{3/2})+m\Tr(\mH)$}
%\xunp{Revision Done}
where the last inequality follows from Lemma~\ref{lem:msq_bound}.
Therefore, we have
\begin{equation}
    \label{ineq:upb_term12_rev}
    \begin{aligned}
        \mathrm{term\ 1.2} \le \left(\frac{L}{m}\right)^2 \cdot \left(1-e^{-mt}\right)^2 \cdot \left(\frac{4(L+m)^2}{\alpha_*}\mathrm{KL}\left(p_0\|p_*\right)+ \frac{2}{L^2}\left(\Tr(\mH^{3/2})+m\Tr(\mH)\right)\right) 
    \end{aligned}
\end{equation}
Combining Eq.~\ref{eq:kld_imp_sde}, Eq.~\ref{ineq:upb_term1}, Eq.~\ref{ineq:upb_term11} and Eq.~\ref{ineq:upb_term12_rev}, we have
\begin{equation*}
    \begin{aligned}
         \frac{\der \mathrm{KL}(\hat{p}_t\|p_*)}{\der t} \le & -\frac{3}{4}\int \hat{p}_t(\vw) \left\|\grad \log \frac{\hat{p}_t(\vw)}{p_*(\vw)}\right\|^2\der \vw + \frac{2(1-e^{-2mt})}{m}\cdot \mathrm{Tr}(\mH)\\
         & + 2\left(\frac{L}{m}\right)^2 \cdot \left(1-e^{-mt}\right)^2 \cdot \left(\frac{4(L+m)^2}{\alpha_*}\mathrm{KL}\left(p_0\|p_*\right)+\frac{2}{L^2}\left(\Tr(\mH^{3/2})+m\Tr(\mH)\right)\right).
    \end{aligned}
\end{equation*}
Then, by LSI of $p_*$, i.e.,~\ref{ass:post_lsi}, it has
\begin{equation}
    \label{ineq:kl_contraction_pre}
    \begin{aligned}
         \frac{\der \mathrm{KL}(\hat{p}_t\|p_*)}{\der t} \le & -\frac{3\alpha_*}{2}\cdot\mathrm{KL}(\hat{p}_t\|p_*) + \frac{2(1-e^{-2mt})}{m}\cdot \mathrm{Tr}(\mH)\\
         & + 2\left(\frac{L}{m}\right)^2 \cdot \left(1-e^{-mt}\right)^2 \cdot \left(\frac{4(L+m)^2}{\alpha_*}\mathrm{KL}\left(p_0\|p_*\right)+\frac{2}{L^2}\left(\Tr(\mH^{3/2})+m\Tr(\mH)\right)\right).
    \end{aligned}
\end{equation}
We wish to integrate the inequality above for $0\le t\le \eta$. 
If we require 
\begin{equation}
    \label{ineq:eta_choice_1}
    \eta \le \min\left\{ (8m)^{-1}, (8\Tr(\mH^{\frac{1}{2}}))^{-1}, (1.5\alpha_*)^{-1}, \frac{\alpha_*}{8\sqrt{2}L^2}\right\},
\end{equation}
then it has
\begin{equation}
    \label{ineq:etabound1}
    \begin{aligned}
        &  \frac{\left(1-e^{-mt}\right)^2}{m^2}  \cdot 4\left(\Tr(\mH^{3/2})+m\Tr(\mH)\right)\\
        \le & \frac{2(1-e^{-2mt})}{m}\cdot \frac{1-e^{-2mt}}{2m} \cdot 4\left(\Tr(\mH^{3/2})+m\Tr(\mH)\right)\\
        \le & \frac{2(1-e^{-2mt})}{m} \cdot 4t \left(\Tr(\mH^{3/2})+m\Tr(\mH)\right)\le \frac{2(1-e^{-2mt})}{m}\cdot \mathrm{Tr}(\mH),
    \end{aligned}
\end{equation}
where the first inequality follows from the monotonicity of $e^{-x}$, the second inequality follows from the fact $(1-e^{-mt})/m \le t$ for any $m\ge 0$, and the last inequality establishes by choice of $\eta$, i.e., Eq.~\ref{ineq:eta_choice_1} and $t\le \eta$.
Plugging Eq.~\ref{ineq:etabound1} into Eq.~\ref{ineq:kl_contraction_pre}, we have
\begin{equation}
    \label{ineq:kl_contraction_mid}
    \begin{aligned}
         \frac{\der \mathrm{KL}(\hat{p}_t\|p_*)}{\der t} \le & -\frac{3\alpha_*}{2}\cdot\mathrm{KL}(\hat{p}_t\|p_*) + \frac{8(L+m)^2L^2}{\alpha_*}\cdot \left(\frac{1-e^{-mt}}{m}\right)^2\cdot \mathrm{KL}\left(p_0\|p_*\right) +\frac{4(1-e^{-2mt})}{m}\cdot \mathrm{Tr}(\mH)\\
         \le & -\frac{3\alpha_*}{2}\cdot\mathrm{KL}(\hat{p}_t\|p_*) + \frac{8(L+m)^2L^2}{\alpha_*}\eta^2\cdot \mathrm{KL}\left(p_0\|p_*\right) + 8\eta \cdot \mathrm{Tr}(\mH).
    \end{aligned}
\end{equation}
Multiplying both sides by $\exp((3\alpha_* t)/2)$, then we have
\begin{equation*}
    \frac{\der}{\der t}\left(e^{\frac{3\alpha_*t}{2}}\mathrm{KL}(\hat{p}_t\|p_*)\right)\le e^{\frac{3\alpha_* t}{2}}\left(\frac{8(L+m)^2L^2}{\alpha_*}\eta^2\cdot \mathrm{KL}\left(p_0\|p_*\right) + 8\eta \cdot \mathrm{Tr}(\mH)\right).
\end{equation*}
By integrating from $t=0$ to $t=\eta$, we have 
\begin{equation*}
    \begin{aligned}
        e^{\frac{3\alpha_*\eta}{2}}\mathrm{KL}(\hat{p}_\eta\|p_*) - \mathrm{KL}(\hat{p}_0\|p_*)\le & \frac{2}{3\alpha_*}\cdot \left(e^{\frac{3\alpha_* \eta}{2}}-1\right)\cdot \left(\frac{8(L+m)^2L^2}{\alpha_*}\eta^2\cdot \mathrm{KL}\left(p_0\|p_*\right) + 8\eta \cdot \mathrm{Tr}(\mH)\right)\\
        \le & 2\eta\cdot \left(\frac{8(L+m)^2L^2}{\alpha_*}\eta^2\cdot \mathrm{KL}\left(p_0\|p_*\right) + 8\eta \cdot \mathrm{Tr}(\mH)\right),
    \end{aligned}
\end{equation*}
where the second inequality follows from 
\begin{equation*}
    e^c\le 1+2c\quad \mathrm{when}\quad c\in[0,1]
\end{equation*}
by choosing $c=(3\alpha_*\eta)/(2)$.
The range of $(3\alpha_*\eta)/(2)$ can be obtained by Eq.~\ref{ineq:eta_choice_1}.
Hence, we have
\begin{equation*}
    \begin{aligned}
        \mathrm{KL}(\hat{p}_{\eta}\|p_*)\le e^{-\frac{3\alpha_* \eta}{2}}\left(1+\frac{16(L+m)L^2}{\alpha_*}\eta^3\right)\mathrm{KL}(\hat{p}_0\|p_*) + e^{-\frac{3\alpha_* \eta}{2}} \cdot 16\eta^2\mathrm{Tr}(\mH).
    \end{aligned}
\end{equation*}
Without loss of generality, we suppose $m\le L$, then because of Eq.~\ref{ineq:eta_choice_1}, we have
\begin{equation*}
    1+\frac{16(L+m)L^2}{\alpha_*}\eta^3 \le 1+ \frac{64L^4}{\alpha_*}\eta^3 \le 1+\frac{\alpha_* \eta}{2}\le e^{\frac{1}{2}\alpha_*\eta}.
\end{equation*}
Combining the previous inequality with the fact $\exp(-(3\alpha_*)/2)\le 1$, it has
\begin{equation}
    \label{ineq:kl_contraction_fin}
    \begin{aligned}
        \mathrm{KL}(\hat{p}_{\eta}\|p_*)\le e^{-\alpha_*\eta}\mathrm{KL}(\hat{p}_0\|p_*) + 16\eta^2\mathrm{Tr}(\mH).
    \end{aligned}
\end{equation}
Hence, the proof is completed.

\subsection{Proof of Theorem~\ref{thm:kl_con_fx}}
Suppose $\eta_k$ is fixed for all iterations, which implies $\tilde{\eta}_{k-1}$ can also be obtained by Eq.~\ref{eq:eta_connection}.
Then, applying the recursion of Eq.~\ref{ineq:kl_contraction_fin}, we have
\begin{equation*}
    \mathrm{KL}(\tilde{p}_k\|p_*)\le e^{-\alpha_*\eta k}\mathrm{KL}(\tilde{p}_0\|p_*) + \frac{16\eta^2\mathrm{Tr}(\mH)}{1-e^{-\alpha_* \eta}}\le e^{-\alpha_*\eta k}\mathrm{KL}(\tilde{p}_0\|p_*) + \frac{32\eta}{\alpha_*}\cdot \mathrm{Tr}(\mH),
\end{equation*}
where the last inequality follows from 
\begin{equation*}
    1-e^{-c}\ge \frac{3}{4}c\quad\mathrm{when}\quad c\in\left[0,\frac{1}{4}\right].
\end{equation*}
By supposing $c=\alpha_*\eta$, we have 
\begin{equation*}
    \eta\le \frac{\alpha_*}{8\sqrt{2}L^2}\le \frac{1}{4\alpha_*} \Rightarrow c\coloneqq \alpha_*\eta \quad \mathrm{and}\quad c\in \left[0,\frac{1}{4}\right].
\end{equation*}
Hence, for any $\epsilon>0$, we set
\begin{equation}
    \label{ineq:eta_choice_fin}
    \eta \le \min\left\{(8m)^{-1}, (8\Tr(\mH^{\frac{1}{2}}))^{-1}, (1.5\alpha_*)^{-1}, \frac{\alpha_*}{8\sqrt{2}L^2}, \frac{\alpha_*\epsilon}{64\mathrm{Tr}(\mH)}\right\},
\end{equation}
then it has
\begin{equation*}
    \mathrm{KL}(\tilde{p}_k\|p_*)\le e^{-\alpha_*\eta k}\mathrm{KL}(\tilde{p}_0\|p_*) + \frac{\epsilon}{2}.
\end{equation*}
In this condition, requiring
\begin{equation*}
    k\ge \frac{1}{\alpha_* \eta}\log \frac{2\mathrm{KL}(\tilde{p}_0\|p_*)}{\epsilon},
\end{equation*}
we have $\mathrm{KL}(\tilde{p}_k\|p_*)\le \epsilon$, and the proof is completed.

\subsection{Proof of Theorem~\ref{thm:kl_con_vary}}
Using the step size $\eta_k = (8\hat{\eta})/9$ where
\begin{equation*}
    \hat{\eta} \coloneqq \min\left\{ (8m)^{-1}, (8\Tr(\mH^{\frac{1}{2}}))^{-1}, (1.5\alpha_*)^{-1}, \frac{\alpha_*}{8\sqrt{2}L^2}\right\},
\end{equation*}
for $k=1, 2,\ldots ,K_0$, we get
\begin{equation}
    \label{ineq:induc_init}
    \begin{aligned}
        \mathrm{KL}(\tilde{p}_k\|p_*)\le &e^{-\alpha_* k \cdot 8\hat{\eta}/9}\mathrm{KL}(\tilde{p}_0\|p_*) + \frac{24 \cdot \left(\frac{8\hat{\eta}}{9}\right)^2\mathrm{Tr}(\mH)}{1-e^{-\alpha_* \cdot 8\hat{\eta}/9}}\\
        \le & e^{-\alpha_* k \cdot 8\hat{\eta}/9}\mathrm{KL}(\tilde{p}_0\|p_*) + \frac{32}{\alpha_*}\cdot \frac{8\hat{\eta}}{9} \cdot \mathrm{Tr}(\mH)\\
        \le & \frac{123\hat{\eta}}{\alpha_*}\cdot \mathrm{Tr}(\mH) + \frac{32}{\alpha_*}\cdot \frac{8\hat{\eta}}{9} \cdot \mathrm{Tr}(\mH) \le \frac{2^9 \cdot \mathrm{Tr}(\mH)}{\alpha_*}\cdot \frac{8\hat{\eta}}{9} = \frac{64\mathrm{Tr}(\mH)}{\alpha_*}\eta_{k+1},
    \end{aligned}
\end{equation}
where the first inequality follows from Theorem~\ref{thm:kl_con_fx}.
Starting from iteration $K_0$,  we use the decreasing step size
\begin{equation*}
    \eta_{k+1} = \frac{8\hat{\eta}}{9+3(k-K_0)c\alpha_*}.
\end{equation*}
Let us show by induction over $k$ that
\begin{equation}
    \label{ineq:cr_varying_ss}
    \mathrm{KL}\left(\tilde{p}_k\|p_*\right)\le \frac{2^9 \cdot \hat{\eta}\cdot \mathrm{Tr}(\mH)}{3\alpha_*\left(3+(k-K_0)c\alpha_*\right)} = \frac{64\mathrm{Tr}(\mH)}{\alpha_*}\eta_{k+1},\quad k\ge K_0.
\end{equation}
For $k = K_0$, this inequality is true in view of Eq.~\ref{ineq:induc_init}.
Assume Eq.~\ref{ineq:cr_varying_ss} establishes at iteration $k$.
For $k+1$, we have
\begin{equation}
    \label{ineq:indunction_varying}
    \begin{aligned}
        \mathrm{KL}\left(\tilde{p}_{k+1}\|p_*\right)\le &\left(1-\frac{3}{4}\alpha_* \eta_{k+1}\right)\mathrm{KL}(\tilde{p}_{k}\|p_*) + 24\eta_{k+1}^2\mathrm{Tr}(\mH)\\
        \le &\left(1-\frac{3}{4}\alpha_* \eta_{k+1}\right)\cdot \frac{64\mathrm{Tr}(\mH)}{\alpha_*}\cdot \eta_{k+1} + 24\eta^2_{k+1}\mathrm{Tr}(\mH)\\
        = & \frac{64\mathrm{Tr}(\mH)}{\alpha_*}\cdot \eta_{k+1} - 24\eta^2_{k+1}\mathrm{Tr}(\mH)\\
        = &\frac{64\mathrm{Tr}(\mH)}{\alpha_*}\eta_{k+1}\cdot\left(1-\frac{3}{8}\alpha_*\eta_{k+1} \right).
    \end{aligned}
\end{equation}
Then, we have
\begin{equation*}
    \begin{aligned}
        & \frac{c\alpha_*}{1+(k-K_0+1)c\alpha_*} \le \frac{c\alpha_*}{1+(k-K_0)c\alpha_*}\\
        \Leftrightarrow\quad & 1-\frac{c\alpha_*}{1+(k-K_0)c\alpha_*}\le \frac{1+(k-K_0)c\alpha_*}{1+(k-K_0+1)c\alpha_*}\\
        \Leftrightarrow\quad & 1-\frac{c\alpha_*}{1+(k-K_0)c\alpha_*}\le \frac{8c}{3+3(k-K_0+1)c\alpha_*}\cdot \frac{3+3(k-K_0)c\alpha_*}{8c}\\
        \Leftrightarrow\quad & \left(1-\frac{3}{8}\alpha_*\eta_{k+1} \right) \eta_{k+1}\le \eta_{k+2}.
    \end{aligned}
\end{equation*}
Plugging this inequality into Eq.~\ref{ineq:indunction_varying}, we have
\begin{equation*}
    \mathrm{KL}\left(\tilde{p}_{k+1}\|p_*\right)\le \frac{64\mathrm{Tr}(\mH)}{\alpha_*}\eta_{k+2}.
\end{equation*}
This completes the proof of the theorem.

\section{Auxiliary lemmas}
\begin{lemma}(Theorem 1.a in~\cite{bellman1968some})
    \label{lem:msq_bound}
    If $\mA \succeq \mB \succeq \vzero$, then $\mA^{\frac{1}{2}}\succeq \mB^{\frac{1}{2}}$, where $\mA^{\frac{1}{2}}$ is semi-positive definite square root of $\mA$.
\end{lemma}
\begin{lemma}
    \label{lem:changeofvar}
    Suppose the joint density function of random variables $\rvx$, $\rvy$ is denoted as $p_{\rvx, \rvy}(\vx,\vy)$, the marginal density function of $\rvx$ and $\rvy$ are denotes as $p_{\rvx}(\vx)$ and $p_{\rvy}(\vy)$ respectively, and the conditional probability density function of $\rvy$ given the occurrence of the value $\vx$ of $\rvx$ can be written as
    \begin{equation*}
        p_{\rvx|\rvy}(\vx|\vy) = \frac{p_{\rvx,\rvy}(\vx,\vy)}{p_{\rvy}(\vy)}.
    \end{equation*}
    Consider another random variable $\rvz$ satisfying $\rvy=\mT(\rvz)$ where $\mT\colon \R^d\rightarrow \R^d$ is bijective and differential.
    We have
    \begin{equation*}
        p_{\rvx|\rvz}(\vx|\vz)=p_{\rvx|\rvy}(\vx|\mT(\vz)).
    \end{equation*}
\end{lemma}
\begin{proof}
    According to the change of variable, we have
    \begin{equation*}
        p_{\rvz}(\vz)=p_{\rvy}(\mT(\vz))\left|\det \left(\grad_{\vz}\mT(\vz)\right)\right|,
    \end{equation*}
    and 
    \begin{equation*}
        p_{\rvx,\rvz}(\vx,\vz)= p_{\rvx,\rvy}(\vx,\mT(\vz))\left|\det \left(\grad_{\vz}\mT(\vz)\right)\right|.
    \end{equation*}
    Then, we have
    \begin{equation*}
        p_{\rvx|\rvz}(\vx|\vz) = \frac{p_{\rvx,\rvz}(\vx,\vz)}{p_{\rvz}(\vz)}=\frac{p_{\rvx,\rvy}(\vx,\mT(\vz))\left|\det \left(\grad_{\vz}\mT(\vz)\right)\right|}{p_{\rvy}(\mT(\vz))\left|\det \left(\grad_{\vz}\mT(\vz)\right)\right|}=p_{\rvx|\rvy}(\vx|\mT(\vz)).
    \end{equation*}
    Hence, the proof is completed.
\end{proof}

\begin{lemma}
    \label{lem:granorm_bound}
    Suppose function $f$ satisfies~\ref{ass:post_f2}, for any $\vw, \vw^\prime\in\R^d$, we have
    \begin{equation*}
        \left\|\grad f(\vw)-\grad f(\vw^\prime)\right\|^2 \le (\vw-\vw^\prime)\mH\left(\vw-\vw^\prime\right).
    \end{equation*}
\end{lemma}
\begin{proof}
    We first define the following two functions
    \begin{equation*}
        \begin{aligned}
            g(t) \coloneqq & \grad f(\vw+ t(\vw^\prime-\vw))\\
            g^\prime(t)\coloneqq & \mH(t)\left(\vw^\prime - \vw\right) = \grad^2 f(\vw+ t(\vw^\prime-\vw))\cdot \left(\vw^\prime - \vw\right)
        \end{aligned}
    \end{equation*}
    By the fundamental theorem of calculus, we have
    \begin{equation*}
        \begin{aligned}
            \grad f(\vw^\prime)- \grad f(\vw) = g(1)-g(0) = \int_0^1 g^\prime(t) \der t = \int_0^1 \mH(t)\der t \cdot \left(\vw^\prime - \vw\right).
        \end{aligned}
    \end{equation*}
    It implies
    \begin{equation*}
        \begin{aligned}
            \left\|\grad f(\vw)-\grad f(\vw^\prime)\right\|^2 = & \left(\vw^\prime - \vw\right)^\top \left(\int_{0}^1 \mH(t) \der t\right)^\top \left(\int_{0}^1 \mH(t) \der t\right)\left(\vw^\prime - \vw\right)\\
            \le & \left(\vw^\prime - \vw\right)^\top \int_{0}^1 \int_{0}^t \left[\mH^\top(t)\mH(t^\prime)+ \mH^\top(t^\prime)\mH(t)\right]\der t^\prime \der t \left(\vw^\prime - \vw\right)\\
            \le & \left(\vw^\prime - \vw\right)^\top \int_{0}^1 \int_{0}^t \left[\mH^\top(t)\mH(t)+ \mH^\top(t^\prime)\mH(t^\prime)\right]\der t^\prime \der t \left(\vw^\prime - \vw\right)\\
            \le & \left(\vw^\prime - \vw\right)^\top \int_{0}^1 \int_{0}^t 2\mH \der t^\prime \der t \left(\vw^\prime - \vw\right)\\
            = & \left(\vw^\prime - \vw\right)^\top \mH\left(\vw^\prime - \vw\right),
        \end{aligned}
    \end{equation*}
    where the second inequality follows from
    \begin{equation*}
        \left(\mH(t)-\mH(t^\prime)\right)^\top\left(\mH(t)-\mH(t^\prime)\right) \succeq \vzero,
    \end{equation*}
    and the last inequality follows from~\ref{ass:post_f2}.
    Hence, the proof is completed.
\end{proof}
\begin{lemma}
    \label{lem:cov_bound}
    For a random (column) vector $\rvw$ with mean vector $\vm=\mathbb{E}(\rvw)$, the covariance matrix is defined as 
    \begin{equation*}
        \mathrm{cov}(\rvw)=\mathbb{E}\left[(\rvw - \vm)(\rvw-\vm)^\top\right].
    \end{equation*}
    Thus, for any fixed $\mQ\colon \R^{d\times d}$, the covariance of $\mA\rvw$, whose mean vector is $\mA\vm$, is given as $\mA\mathrm{cov}(\rvw)\mA^\top$.
\end{lemma}
\begin{proof}
    \begin{equation*}
        \begin{aligned}
            \mathrm{cov}\left(\mA \rvw\right) = &\mathbb{E}\left[(\mA\rvw - \mA\vm)(\mA\rvw-\mA\vm)^\top\right]\\
            = & \mathbb{E}\left[\mA(\rvw - \mA\vm)(\rvw-\vm)^\top\mA^\top\right]\\
            = & \mA \mathbb{E}\left[(\rvw - \vm)(\rvw-\vm)^\top\right] \mA^\top\\
            = & \mA \mathrm{cov}(\rvw) \mA^\top.
        \end{aligned}
    \end{equation*}
    Hence, the proof is completed.
\end{proof}

\begin{lemma}
    \label{lem:rapid11}
    Assume the density function $p_*=e^{-U}$, where $U$ can be decomposed as Eq.~\ref{def:tar_dis_com},  satisfies~\ref{ass:post_f} and~\ref{ass:post_f2}. 
    Then, we have
    \begin{equation*}
        \int p_*(\vw) \left\|\grad U(\vw)\right\|^2\der\vw = \int p_*(\vw) \left\|\grad f(\vw)+m\vw\right\|^2\der\vw \le \mathrm{Tr}\left(\mH^{\frac{1}{2}}\right)+md.
    \end{equation*}
\end{lemma}
\begin{proof}
    Since $p_*=e^{-U}$, due to the integration by parts we have
    \begin{equation}
        \label{eq:igbp_bound}
        \mathbb{E}_{p_*}\left[\left\|\grad U\right\|^2\right] = \mathbb{E}_{p_*}\left[\Delta U\right].
    \end{equation}
    For each $\vw\in\R^d$, we suppose $\grad^2 U(\vw) = \mV\Sigma\mV^\top$, where the $i$-th column vector $\vv_i$ of $\mV$ denotes the $i$-th eigenvector and $\sigma_{i}$ denotes the corresponding eigenvalue satisfying
    \begin{equation*}
        \sigma_{1}\ge \sigma_{2}\ge \ldots \ge \sigma_{r}\ge m \ge \sigma_{(r+1)}\ge \ldots \ge \sigma_{d}.
    \end{equation*}
    According to the decomposition of $U$, we have the following inequality
    \begin{equation*}
        \begin{aligned}
            \grad^2 f(\vw) = &\grad^2 U(\vw) - m\mI = \mV \left(\Sigma -m\mI\right)\mV^\top\\
            \preceq &\mV\mathrm{diag}\left\{(\sigma_{1}-m),(\sigma_{2}-m),\ldots,(\sigma_{r}-m), 0,\ldots, 0\right\}\mV^\top \preceq \mH^{\frac{1}{2}},
        \end{aligned}
    \end{equation*}
    where the last inequality follows from
    \begin{equation*}
        \begin{aligned}
            \sum_{i=1}^r(\sigma_i -m)^2 \vv_i\vv_i^\top \preceq \grad^2 f(\vw)\left(\grad^2 f(\vw) \right)^\top = \sum_{i=1}^d (\sigma_i-m)^2\vv_i\vv_i^\top\preceq \mH,
        \end{aligned}
    \end{equation*}
    and Lemma~\ref{lem:msq_bound}.
    In this condition, we have
    \begin{equation*}
        \mathrm{Tr}(\grad^2 U(\vw))\le \mathrm{Tr}\left(\mH^{\frac{1}{2}}+m\mI\right) = \mathrm{Tr}\left(\mH^{\frac{1}{2}}\right)+md.
    \end{equation*}
    Combining this inequality with Eq.~\ref{eq:igbp_bound}, we can complete the proof.
\end{proof}
\begin{lemma}
    \label{lem:rapid12}
    Assume $p_*=e^{_U}$, where $U$ can be decomposed as Eq.~\ref{def:tar_dis_com}, satisfies~\ref{ass:post_f}-\ref{ass:post_f2}.
    It implies the establishment of Talagrand's inequality with constant $\alpha_*$, and $L+m$ smooth.
    For any $p$, 
    \begin{equation*}
        \int p(\vw)\left\|\grad U(\vw)\right\|^2\der\vw \le \frac{4(L+m)^2}{\alpha_*}\mathrm{KL}\left(p\|p_*\right)+ 2\mathrm{Tr}\left(\mH^{\frac{1}{2}}\right)+2md.
    \end{equation*}
\end{lemma}
\begin{proof}
    Suppose $\rvw\sim p$ and $\rvw_*\sim p_*$ with an optimal coupling $(\vw, \vw_*)\sim \gamma$ and $W_2^2(p,p_*) = \mathbb{E}_\gamma\left[\left\|\rvw - \rvw_*\right\|^2\right]$.
    Due to the $L+m$ smoothness of $U$, we have
    \begin{equation*}
        \begin{aligned}
            \left\|\grad U(\vw)\right\| = & \left\|\grad U(\vw)- \grad U(\vw_*)\right\|+\left\|\grad U(\vw_*)\right\|\\
            \le & (L+m)\left\|\vw-\vw_*\right\|+\left\|\grad U(\vw_*)\right\|
        \end{aligned}
    \end{equation*}
    Taking expectations, we have
    \begin{equation*}
        \begin{aligned}
            \mathbb{E}_p\left[\left\|\grad U(\rvw)\right\|^2\right]\le & 2(L+m)^2 \mathbb{E}_\gamma\left[\left\|\rvw - \rvw_*\right\|^2\right] + 2\mathbb{E}_{p_*}\left[\left\|\grad U(\rvw_*)\right\|^2\right]\\
            \le & 2(L+m)^2 W_2^2(p,p_*) + 2\mathrm{Tr}\left(\mH^{\frac{1}{2}}\right)+2md\\
            \le & \frac{4(L+m)^2}{\alpha_*}\mathrm{KL}\left(p\|p_*\right)+ 2\mathrm{Tr}\left(\mH^{\frac{1}{2}}\right)+2md,
        \end{aligned}
    \end{equation*}
    where the second inequality follows from Lemma~\ref{lem:rapid11} and the third inequality follows from Talagrand's inequality.
\end{proof}

\section{Background on Kolmogorov equations and transition density}
\label{appsec:bke}
Our analysis relies on the theory of Markov diffusion operators. 
In this section of the Appendix, we summarize some key ideas and results; the book~\cite{bakry2014analysis} and~\cite{sarkka2019applied} provide a more detailed exposition.

Suppose $\left\{\rvw_t\right\}_{t\ge 0}$ be a continous-time homogeneous Markov process with values in $\R^d$. 
Then, we define one Markov semigroup as
\begin{equation}
    \label{def:msg}
    \gP_t \phi(\vw)= \mathbb{E}\left[\phi(\rvw_t|\rvw_0=\vw\right]
\end{equation}
for all bounded measurable function $\phi \colon \R^d \rightarrow \R$.
Such a semigroup can be driven by an infinitesimal operator defined as
\begin{equation}
    \label{def:ifo}
    \gL \phi(\vw) = \lim_{\tau \rightarrow 0} \frac{\gP_{\tau}\phi(\vw) - \phi(\vw)}{\tau}.
\end{equation}
Due to the semigroup property, we have
\begin{equation*}
    \begin{aligned}
        \partial_t \gP_t & = \lim_{\tau \rightarrow 0} \frac{\gP_{t+\tau} - \gP_{t}}{\tau}=\gP_{t} \left[\lim_{\tau\rightarrow 0}\frac{\gP_\tau -\gI}{\tau}\right] = \gP_{t}\gL\\
        & =\left[\lim_{\tau\rightarrow 0}\frac{\gP_\tau -\gI}{\tau}\right] \gP_\tau = \gL \gP_t,
    \end{aligned}
\end{equation*}
where the second and the fourth equation follows from the linear property of semigroup $\gP_t$.

Specifically, for an It\^o process that solves the following SDE
\begin{equation}
    \label{def:itosde}
    \der \rvw_t = \vmu(\rvw_t,t)\der t + \vsigma(\rvw_t, t)\der \mB_t,
\end{equation}
where $\left\{\mB_t\right\}_{t\ge 0}$ is the standard Brownian motion in $\R^d$ with $\mB_t = \vzero$, the generator of $\rvw_t$ is given as
\begin{equation}
    \label{def:ifosde}
    \gL \phi(\vw) =\sum_{i} \frac{\partial \phi(\vw)}{\partial \vw_i} \vmu_i(\vw) + \frac{1}{2}\sum_{i,j}\left[\frac{\partial^2 \phi(\vw)}{\partial \vw_i \partial \vw_j}\right]\cdot \left[\vsigma(\vw)\vsigma(\vw)^\top\right]_{ij}
\end{equation}
by definition Eq.~\ref{def:ifo}.
Here, we suppose $\rvw_t$ is time-invariant, which implies $\vmu$ and $\vsigma$ in SDE.~\ref{def:itosde} will not explicitly depend on time $t$.

Then, we will establish the connection between $\gL$ and forward Kolmogorov equation (Fokker-Planck equation) with an operator formulation.
Suppose $\rvw_t \sim p_t$, with a bit of abuse of notation, $p_t$ denotes the density function.
In this condition, the expectation of a function $\phi$ at time $t$ is given as
\begin{equation*}
    \mathbb{E}\left[\phi (\rvw_t)\right] = \left<\phi, p_t\right> \coloneqq \int \phi(\vw) p_t(\vw)\der\vw,
\end{equation*}
which implies
\begin{equation*}
    \begin{aligned}
        \frac{\der \left<\phi, p_t\right>}{\der t} = & \frac{\der}{\der t}\int p_0(\vw_0) \gP_t \phi(\vw_0)\der\vw_0 =\int p_0(\vw_0) \gP_t\gL \phi(\vw_0)\der \vw_0= \left<\gL\phi, p_t\right>.
    \end{aligned}
\end{equation*}
Note that the second equation follows from the definition of $\gL$.
Using the adjoint operator $\gL^*$, we can write the previous equation as
\begin{equation}
    \label{eq:fpof}
    \frac{\der \left<\phi, p_t\right>}{\der t} = \left<\phi, \frac{\partial p_t}{\partial t}\right> = \left<\gL\phi, p_t\right> = \left<\phi, \gL^*p_t\right>.
\end{equation}
Due to the arbitrary of $\phi$, Eq.~\ref{eq:fpof} can only be true if 
\begin{equation}
    \label{def:fpe}
    \frac{\partial p_t}{\partial t}=\gL^* p_t.
\end{equation}
For the infinitesimal operator Eq.~\ref{def:ifosde} of It\^o SDE, its adjoint operator is given as
\begin{equation}
    \label{def:ajifosde}
    \gL^* \phi(\vw)= -\sum_i \frac{\partial}{\partial \vw_i}\left(\vmu_i(\vw) \phi(\vw)\right) + \frac{1}{2}\sum_{i,j}\frac{\partial^2}{\partial \vw_i\partial \vw_j}\left(\left[\vsigma(\vw)\vsigma(\vw)^\top\right]_{ij}\right).
\end{equation}
In this condition, Eq.~\ref{def:fpe}  is exactly the operator appearing in the forward Kolmogorov equation.

For all It\^o processes, the solution to the corresponding SDE.~\ref{def:itosde}, satisfies the Markov property, which means
\begin{equation*}
    p_t\left(\cdot | \left\{\rvw_s\right\}_{s\in[t,0]}\right) =p_t\left(\cdot| \rvw_s\right).
\end{equation*}
This result can be found, for example, in~\cite{oksendal2013stochastic}.
Hence, in a probability sense, the transition density $p(\vw_t, t | \vw_s, s)$ not only characterizes It\^o processes but also be a solution to the forward Kolmogorov equation with a degenerate (Dirac delta) initial density concentrated on $\vw_s$ at time s.
\begin{lemma}(Theorem 5.10 in~\cite{sarkka2019applied})
    \label{lem:fpe}
    The transition density $p(\vw_t, t|\vw_s, s)$ of SDE.~\ref{def:itosde}, where $t\ge s$ is the solution of forward Komolgorov equation~\ref{def:fpe} with the initial condition $p(\vw_t,t|\vw_s,s) = \delta(\vw_t-\vw_s)$ at $t=s$. It means
    \begin{equation*}
        \frac{\partial p(\vw_t, t|\vw_s, s)}{\partial t}=\gL^* p(\vw_t, t|\vw_s, s),\quad p(\vw^\prime,s|\vw,s)=\delta(\vw^\prime - \vw).
    \end{equation*}
\end{lemma}
Interestingly, the transition density also satisfies another equation, the backward Kolmogorov equation, as follows
\begin{lemma}(Theorem 5.11 in ~\cite{sarkka2019applied})
    \label{lem:kbe}
    The transition density $p(\vw_t, t|\vw_s, s)$ of SDE.~\ref{def:itosde}, where $t\ge s$ is the solution of backward Komolgorov equation
    \begin{equation*}
        -\frac{\partial p(\vw_t, t|\vw_s, s)}{\partial s}=\gL p(\vw_t, t|\vw_s, s),\quad p(\vw^\prime,s|\vw,s)=\delta(\vw^\prime - \vw).
    \end{equation*}
\end{lemma}
\section{Proof of Equation~\ref{def:kbe_equ}}
\label{sec:par_cal}
Following from Appendix~\ref{sec:implictsde}, we suppose
\begin{equation*}
    \small
    \begin{aligned}
        p_{\vw_0}(\vw, t | \vw^\prime, t^\prime) 
        = \left(\frac{2\pi (1-e^{-2m(t-t^\prime)})}{m}\right)^{-d/2}\cdot \exp \left[-\frac{m}{2}\cdot \frac{\left\|\vw-\vw^\prime e^{-m(t-t^\prime)}+ \left(1-e^{-m(t-t^\prime)}\right)\cdot \left(\frac{\tilde{\eta}}{e^{m\eta} -1}\right)\grad f(\vw_0)\right\|^2}{1-e^{-2m(t-t^\prime)}}\right].
    \end{aligned}
\end{equation*}
For abbreviation, we set
\begin{equation*}
    \begin{aligned}
        l_1(\vw^\prime, t)= & \left(\frac{2\pi (1-e^{-2m(t-t^\prime)})}{m}\right)^{-d/2}\\
        l_2(\vw^\prime, t)= & \exp \left[-\frac{m}{2}\cdot \frac{\left\|\vw-\vw^\prime e^{-m(t-t^\prime)}+ \left(1-e^{-m(t-t^\prime)}\right)\cdot \left(\frac{\tilde{\eta}}{e^{m\eta} -1}\right)\grad f(\vw_0)\right\|^2}{1-e^{-2m(t-t^\prime)}}\right].
    \end{aligned}
\end{equation*}
Considering the partial derivative of $l_1$ w.r.t. $t^\prime$, we have
\begin{equation}
    \label{eq:par_1}
    \frac{\partial \left(1-e^{-2m(t-t^\prime)}\right)}{\partial t^\prime} = (-2m)\cdot e^{-2m(t-t^\prime)}.
\end{equation}
By the chain rule, we have
\begin{equation}
    \label{eq:l1pd}
    \begin{aligned}
        \frac{\partial l_1(\vw^\prime, t^\prime)}{\partial t^\prime} = & -\frac{d}{2}\cdot \left(\frac{2\pi (1-e^{-2m(t-t^\prime)})}{m}\right)^{-d/2-1}\cdot \frac{2\pi}{m}\cdot \frac{\partial \left(1-e^{-2m(t-t^\prime)}\right)}{\partial t^\prime}\\
        =& 2\pi d\cdot \left(\frac{2\pi (1-e^{-2m(t-t^\prime)})}{m}\right)^{-d/2-1}\cdot e^{-2m(t-t^\prime)}.
    \end{aligned}
\end{equation}
Then, we consider the partial derivative of $l_2$ w.r.t. $t^\prime$, and set
\begin{equation*}
    \begin{aligned}
        l_{2.1}(\vw^\prime, t^\prime) = & \left\|\vw-\vw^\prime e^{-m(t-t^\prime)}+ \left(1-e^{-m(t-t^\prime)}\right)\cdot \left(\frac{\tilde{\eta}}{e^{m\eta} -1}\right)\grad f(\vw_0)\right\|^2\\
        l_{2.2}\left(\vw^\prime, t^\prime\right)=& 1-e^{-2m(t-t^\prime)}\\
        \Rightarrow\quad  l_2(\vw^\prime, t^\prime)=&\exp\left[-\frac{m}{2}\cdot \frac{l_{2.1}(\vw^\prime, t^\prime)}{l_{2.2}(\vw^\prime, t^\prime)}\right].
    \end{aligned}
\end{equation*}
We have
\begin{equation}
    \label{eq:2.1pd}
    \begin{aligned}
        \frac{\partial l_{2.1}(\vw^\prime, t^\prime)}{\partial t^\prime}=& 2\left(\vw-\vw^\prime e^{-m(t-t^\prime)}+ \left(1-e^{-m(t-t^\prime)}\right)\cdot \left(\frac{\tilde{\eta}}{e^{m\eta} -1}\right)\grad f(\vw_0)\right) \\
        &\cdot\frac{\partial \left(\vw-\vw^\prime e^{-m(t-t^\prime)}+ \left(1-e^{-m(t-t^\prime)}\right)\cdot \left(\frac{\tilde{\eta}}{e^{m\eta} -1}\right)\grad f(\vw_0)\right)}{\partial t^\prime}\\
        = & 2 \left(\vw-\vw^\prime e^{-m(t-t^\prime)}+ \left(1-e^{-m(t-t^\prime)}\right)\cdot \left(\frac{\tilde{\eta}}{e^{m\eta} -1}\right)\grad f(\vw_0)\right) \\
        & \cdot\left(-e^{-m(t-t^\prime)}m\cdot \vw^\prime - e^{-m(t-t^\prime)}\cdot m \cdot \frac{\tilde{\eta}}{e^{m\eta}-1}\cdot \grad f(\vw_0)\right)\\
        = & -2m e^{-m(t-t^\prime)} \left(\vw-\vw^\prime e^{-m(t-t^\prime)}+ \left(1-e^{-m(t-t^\prime)}\right) \left(\frac{\tilde{\eta}}{e^{m\eta} -1}\right)\grad f(\vw_0)\right)\\
        &\cdot\left(\vw^\prime + \frac{\tilde{\eta}}{e^{m\eta} -1} \grad f(\vw_0)\right).
    \end{aligned}
\end{equation}
partial derivative of $l_{2.2}$ w.r.t. $t^\prime$ is provided by Eq.~\ref{eq:par_1}, i.e.,
\begin{equation}
    \label{eq:2.2pd}
    \frac{\partial l_{2.2}(\vw^\prime, t^\prime)}{\partial t^\prime}= (-2m)\cdot e^{-2m(t-t^\prime)}.
\end{equation}
In this condition, we have
\begin{equation*}
    \begin{aligned}
        \frac{\partial l_{2}(\vw^\prime, t^\prime)}{\partial t^\prime} = &\exp\left[-\frac{m}{2}\cdot \frac{l_{2.1}(\vw^\prime, t^\prime)}{l_{2.2}(\vw^\prime, t^\prime)}\right]\cdot \frac{-m}{2}\cdot \left(l^{-2}_{2.2}(\vw^\prime, t^\prime)\right)\cdot\\
        &\left(\left(\frac{l_{2.1}(\vw^\prime, t^\prime)}{\partial \vw^\prime}\right)\cdot l_{2.2}(\vw^\prime, t^\prime)-l_{2.1}(\vw^\prime, t^\prime)\left(\frac{\partial l_{2.2}(\vw^\prime, t^\prime)}{\partial \vw^\prime}\right)\right)\\
        =& \exp\left[-\frac{m}{2}\cdot \frac{l_{2.1}(\vw^\prime, t^\prime)}{l_{2.2}(\vw^\prime, t^\prime)}\right]\cdot \frac{-m}{2}\cdot \left(\frac{l_{2.1}(\vw^\prime, t^\prime)}{\partial \vw^\prime}\right)\cdot l^{-1}_{2.2}(\vw^\prime, t^\prime)\\
        & + \exp\left[-\frac{m}{2}\cdot \frac{l_{2.1}(\vw^\prime, t^\prime)}{l_{2.2}(\vw^\prime, t^\prime)}\right]\cdot \frac{m}{2}\cdot \frac{l_{2.1}(\vw^\prime, t^\prime)}{l^2_{2.2}(\vw^\prime, t^\prime)}\cdot \frac{\partial l_{2.2}(\vw^\prime, t^\prime)}{\partial \vw^\prime}
    \end{aligned}
\end{equation*}
Plugging Eq.~\ref{eq:2.1pd} and Eq.~\ref{eq:2.2pd} into the previous equation, we have
\begin{equation}
    \small
    \label{eq:l2pd}
    \begin{aligned}
        \frac{\partial l_2(\vw^\prime, t^\prime)}{\partial t^\prime} =& \exp \left(-\frac{m}{2}\cdot \frac{\left\|\vw-\vw^\prime e^{-m(t-t^\prime)}+ \left(1-e^{-m(t-t^\prime)}\right)\cdot \left(\frac{\tilde{\eta}}{e^{m\eta} -1}\right)\grad f(\vw_0)\right\|^2}{1-e^{-2m(t-t^\prime)}}\right)  \cdot \frac{m^2 e^{-m(t-t^\prime)}}{1-e^{-2m(t-t^\prime)}}\\
        & \cdot \left(\vw-\vw^\prime e^{-m(t-t^\prime)}+ \left(1-e^{-m(t-t^\prime)}\right)\cdot \left(\frac{\tilde{\eta}}{e^{m\eta} -1}\right)\grad f(\vw_0)\right) \cdot \left(\vw^\prime + \frac{\tilde{\eta}}{e^{m\eta} -1} \grad f(\vw_0)\right)\\
        & -\exp \left(-\frac{m}{2}\cdot \frac{\left\|\vw-\vw^\prime e^{-m(t-t^\prime)}+ \left(1-e^{-m(t-t^\prime)}\right)\cdot \left(\frac{\tilde{\eta}}{e^{m\eta} -1}\right)\grad f(\vw_0)\right\|^2}{1-e^{-2m(t-t^\prime)}}\right)  \cdot \frac{m^2 e^{-2m(t-t^\prime)}}{(1-e^{-2m(t-t^\prime)})^2}\\
        &\cdot \left\|\vw-\vw^\prime e^{-m(t-t^\prime)}+ \left(1-e^{-m(t-t^\prime)}\right)\cdot \left(\frac{\tilde{\eta}}{e^{m\eta} -1}\right)\grad f(\vw_0)\right\|^2.
    \end{aligned}
\end{equation}
Combining Eq.~\ref{eq:l1pd} and Eq.~\ref{eq:l2pd}, we have
\begin{equation}
    \small
    \label{eq:pdt}
    \begin{aligned}
        &\frac{\partial p_{\vw_0}(\vw,t|\vw^\prime, t^\prime)}{\partial t^\prime} = \frac{\partial (l_1(\vw^\prime , t^\prime)l_2(\vw^\prime, t^\prime))}{\partial t^\prime}=\frac{\partial l_1(\vw^\prime, t^\prime)}{\partial t^\prime}\cdot l_2(\vw^\prime, t^\prime) + l_1(\vw^\prime, t^\prime)\cdot \frac{\partial l_2(\vw^\prime, t^\prime)}{\partial t^\prime}\\
        =& \exp \left[-\frac{m}{2}\cdot \frac{\left\|\vw-\vw^\prime e^{-m(t-t^\prime)}+ \left(1-e^{-m(t-t^\prime)}\right)\cdot \left(\frac{\tilde{\eta}}{e^{m\eta} -1}\right)\grad f(\vw_0)\right\|^2}{1-e^{-2m(t-t^\prime)}}\right]\\
        &\cdot 2\pi d\cdot \left(\frac{2\pi (1-e^{-2m(t-t^\prime)})}{m}\right)^{-d/2-1}\cdot e^{-2m(t-t^\prime)}\\
        & +  \exp \left[-\frac{m}{2}\cdot \frac{\left\|\vw-\vw^\prime e^{-m(t-t^\prime)}+ \left(1-e^{-m(t-t^\prime)}\right)\cdot \left(\frac{\tilde{\eta}}{e^{m\eta} -1}\right)\grad f(\vw_0)\right\|^2}{1-e^{-2m(t-t^\prime)}}\right] \\
        & \cdot \left(\frac{2\pi (1-e^{-2m(t-t^\prime)})}{m}\right)^{-d/2}\cdot \left[\frac{m^2 e^{-m(t-t^\prime)}}{1-e^{-2m(t-t^\prime)}}\cdot \left(\vw-\vw^\prime e^{-m(t-t^\prime)}+ \left(1-e^{-m(t-t^\prime)}\right)\cdot \left(\frac{\tilde{\eta}}{e^{m\eta} -1}\right)\grad f(\vw_0)\right)\right. \\
        &\left.\cdot \left(\vw^\prime + \frac{\tilde{\eta}}{e^{m\eta} -1} \grad f(\vw_0)\right) -\frac{m^2 e^{-2m(t-t^\prime)}}{(1-e^{-2m(t-t^\prime)})^2} \cdot \left\|\vw-\vw^\prime e^{-m(t-t^\prime)}+ \left(1-e^{-m(t-t^\prime)}\right)\cdot \left(\frac{\tilde{\eta}}{e^{m\eta} -1}\right)\grad f(\vw_0)\right\|^2\right].
    \end{aligned}
\end{equation}
Next, we consider the partial derivative of $l_1$ w.r.t. $\vw^\prime$.
We have
\begin{equation*}
    \frac{\partial l_{2.1}(\vw^\prime, t^\prime)}{\partial \vw^\prime} = \vw-\vw^\prime e^{-m(t-t^\prime)}+ \left(1-e^{-m(t-t^\prime)}\right)\cdot \left(\frac{\tilde{\eta}}{e^{m\eta} -1}\right)\grad f(\vw_0)
\end{equation*}
which means
\begin{equation}
    \label{eq:l2pdw}
    \begin{aligned}
        &\frac{\partial p_{\vw_0}(\vw,t|\vw^\prime, t^\prime)}{\partial \vw^\prime}= l_1(\vw^\prime, t^\prime)\cdot \frac{\partial l_2(\vw^\prime, t^\prime)}{\partial \vw^\prime} = l_1(\vw^\prime, t^\prime)\cdot l_2(\vw^\prime, t^\prime)\cdot \frac{-m}{2(1-e^{-2m(t-t^\prime)})}\cdot \frac{\partial l_{2.1}(\vw^\prime, t^\prime)}{\partial \vw^\prime}\\
        = & \left(\frac{2\pi (1-e^{-2m(t-t^\prime)})}{m}\right)^{-d/2}\cdot  \exp \left[-\frac{m}{2}\cdot \frac{\left\|\vw-\vw^\prime e^{-m(t-t^\prime)}+ \left(1-e^{-m(t-t^\prime)}\right)\cdot \left(\frac{\tilde{\eta}}{e^{m\eta} -1}\right)\grad f(\vw_0)\right\|^2}{1-e^{-2m(t-t^\prime)}}\right]\\
        & \cdot \frac{m e^{-m(t-t^\prime)}}{1-e^{-2m(t-t^\prime)}}\cdot\left(\vw-\vw^\prime e^{-m(t-t^\prime)}+ \left(1-e^{-m(t-t^\prime)}\right)\cdot \left(\frac{\tilde{\eta}}{e^{m\eta} -1}\right)\grad f(\vw_0)\right).
    \end{aligned}
\end{equation}
Then, we can easily obtain that
\begin{equation*}
    \begin{aligned}
        &\frac{\partial^2 p_{\vw_0}(\vw,t|\vw^\prime, t^\prime)}{\partial \vw_i^\prime \partial\vw_i^\prime}= \partial \left(l_1(\vw^\prime, t^\prime)\cdot  l_2(\vw^\prime, t^\prime)\cdot \frac{-m}{2(1-e^{-2m(t-t^\prime)})}\cdot \frac{\partial l_{2.1}(\vw^\prime, t^\prime)}{\partial \vw^\prime}\right)_i \Big/ \partial \vw^\prime_i\\
        =& \left(\frac{2\pi (1-e^{-2m(t-t^\prime)})}{m}\right)^{-d/2}\cdot \frac{m e^{-m(t-t^\prime)}}{1-e^{-2m(t-t^\prime)}}\cdot \left[\left(-e^{-m(t-t^\prime)}\right)\cdot l_2(\vw^\prime, t^\prime)\right.\\
        &\left. +\frac{\partial l_2(\vw^\prime, t^\prime)}{\partial \vw^\prime_i}\cdot \left(\vw-\vw^\prime e^{-m(t-t^\prime)}+ \left(1-e^{-m(t-t^\prime)}\right)\cdot \left(\frac{\tilde{\eta}}{e^{m\eta} -1}\right)\grad f(\vw_0)\right)_i\right]\\
        = & \left(\frac{2\pi (1-e^{-2m(t-t^\prime)})}{m}\right)^{-d/2}\cdot \frac{me^{-m(t-t^\prime)}}{1-e^{-2m(t-t^\prime)}}\cdot\\
        &\exp \left(-\frac{m}{2}\cdot \frac{\left\|\vw-\vw^\prime e^{-m(t-t^\prime)}+ \left(1-e^{-m(t-t^\prime)}\right)\cdot \left(\frac{\tilde{\eta}}{e^{m\eta} -1}\right)\grad f(\vw_0)\right\|^2}{1-e^{-2m(t-t^\prime)}}\right)\cdot (-e^{-m(t-t^\prime)})\\
        & + \left(\frac{2\pi (1-e^{-2m(t-t^\prime)})}{m}\right)^{-d/2}\cdot \exp \left(-\frac{m}{2}\cdot \frac{\left\|\vw-\vw^\prime e^{-m(t-t^\prime)}+ \left(1-e^{-m(t-t^\prime)}\right)\cdot \left(\frac{\tilde{\eta}}{e^{m\eta} -1}\right)\grad f(\vw_0)\right\|^2}{1-e^{-2m(t-t^\prime)}}\right)\\
        & \cdot \frac{m^2 e^{-2m(t-t^\prime)}}{\left(1-e^{-2m(t-t^\prime)}\right)^2}\cdot\left(\vw-\vw^\prime e^{-m(t-t^\prime)}+ \left(1-e^{-m(t-t^\prime)}\right)\cdot \left(\frac{\tilde{\eta}}{e^{m\eta} -1}\right)\grad f(\vw_0)\right)_i^2.\\
    \end{aligned}
\end{equation*}
In this condition, we have
\begin{equation*}
    \label{eq:kbe_itosde}
    \begin{aligned}
        &\left(-m\vw^\prime - \frac{m\tilde{\eta}}{e^{m\eta}-1} \grad f(\vw_0)\right)\cdot \frac{\partial p_{\vw_0}(\vw, t|\vw^\prime, t^\prime)}{\partial \vw^\prime} + 1\cdot \sum_{i=1}^d \frac{\partial^2 p(\vw,t | \vw^\prime, t^\prime ,\vw_0)}{\left(\partial \vw_i\right)^2}\\
        = & -\left(\frac{2\pi (1-e^{-2m(t-t^\prime)})}{m}\right)^{-d/2}\cdot  \exp \left[-\frac{m}{2}\cdot \frac{\left\|\vw-\vw^\prime e^{-m(t-t^\prime)}+ \left(1-e^{-m(t-t^\prime)}\right)\cdot \left(\frac{\tilde{\eta}}{e^{m\eta} -1}\right)\grad f(\vw_0)\right\|^2}{1-e^{-2m(t-t^\prime)}}\right]\\
        & \cdot \frac{m e^{-m(t-t^\prime)}}{1-e^{-2m(t-t^\prime)}}\cdot\left(\vw-\vw^\prime e^{-m(t-t^\prime)}+ \left(1-e^{-m(t-t^\prime)}\right)\cdot \left(\frac{\tilde{\eta}}{e^{m\eta} -1}\right)\grad f(\vw_0)\right)\cdot \left(m\vw^\prime + \frac{\tilde{m\eta}}{e^{m\eta} -1} \grad f(\vw_0)\right)\\
        &+\left(\frac{2\pi (1-e^{-2m(t-t^\prime)})}{m}\right)^{-d/2}\cdot \exp \left(-\frac{m}{2}\cdot \frac{\left\|\vw-\vw^\prime e^{-m(t-t^\prime)}+ \left(1-e^{-m(t-t^\prime)}\right)\cdot \left(\frac{\tilde{\eta}}{e^{m\eta} -1}\right)\grad f(\vw_0)\right\|^2}{1-e^{-2m(t-t^\prime)}}\right)\\
        & \cdot \frac{m^2 e^{-2m(t-t^\prime)}}{\left(1-e^{-2m(t-t^\prime)}\right)^2}\cdot\left\|\vw-\vw^\prime e^{-m(t-t^\prime)}+ \left(1-e^{-m(t-t^\prime)}\right)\cdot \left(\frac{\tilde{\eta}}{e^{m\eta} -1}\right)\grad f(\vw_0)\right\|^2\\
        & + \left(\frac{2\pi (1-e^{-2m(t-t^\prime)})}{m}\right)^{-d/2}\cdot \exp \left(-\frac{m}{2}\cdot \frac{\left\|\vw-\vw^\prime e^{-m(t-t^\prime)}+ \left(1-e^{-m(t-t^\prime)}\right)\cdot \left(\frac{\tilde{\eta}}{e^{m\eta} -1}\right)\grad f(\vw_0)\right\|^2}{1-e^{-2m(t-t^\prime)}}\right)\\
        &\cdot d\cdot \frac{me^{-m(t-t^\prime)}}{1-e^{-2m(t-t^\prime)}}\cdot (-e^{-m(t-t^\prime)})\\
        = & -\frac{\partial p_{\vw_0}(\vw,t|\vw^\prime, t^\prime)}{\partial t^\prime},
    \end{aligned}
\end{equation*}
where the last equation can be obtained by comparing the result with Eq.~\ref{eq:pdt}
Hence, Eq.~\ref{def:kbe_equ} is established.

\end{document}